\documentclass[11pt]{article}
\usepackage{latexsym,amssymb,amsmath} % for \Box, \mathbb, split, etc.
\usepackage{path}
\usepackage{url}
\usepackage{verbatim}

\usepackage{natbib}
\usepackage{amsfonts}
\usepackage{amsfonts}
\usepackage{amsfonts}
\usepackage{array}
\usepackage{mathrsfs}
\usepackage{amsfonts}
\usepackage{amsmath}
\usepackage{amssymb}
\usepackage{amsthm}
\usepackage{bm}
\usepackage{appendix}

\usepackage{graphicx} % more modern
\usepackage{subfigure}

\graphicspath{{figures/}}

\usepackage[lined,boxed,ruled, commentsnumbered]{algorithm2e}

\newtheorem{lemma}{Lemma}
\newtheorem{theorem}{Theorem}

\newtheorem{proposition}{Proposition}
\newtheorem{definition}{Definition}
\newtheorem{remark}{Remark}
\newtheorem{assumption}{Assumption}

\newcommand{\supp}{\text{supp}}

\newcommand{\argmin}{\mathop{\arg\min}}
\newcommand{\argmax}{\mathop{\arg\max}}

\newcommand{\Tr}{\text{Tr}}

\newcommand{\st}{\text{s.t. }}
\newcommand{\diag}{\text{diag}}

\numberwithin{equation}{section}
\numberwithin{table}{section}
\numberwithin{figure}{section}

\newcommand{\junk}[1]{{}}
\newcommand{\vect}{\text{vect} }
\newlength{\fwtwo} 
\DeclareGraphicsExtensions{.jpg}

\title{Gradient Hard Thresholding Pursuit for Sparsity-Constrained Optimization}
\author{
  Xiao-Tong Yuan$^{1,2}$,\ \  Ping Li$^{2,3}$, \ \ Tong Zhang$^{2}$\\\\
  1. Department of Statistical Science, Cornell University \\
  Ithaca, New York, 14853, USA \\
  \and
  2. Department of Statistics \& Biostatistics, Rutgers University \\
  Piscataway, New Jersey, 08854, USA \\
  \and
  3. Department of Computer Science, Rutgers University \\
  Piscataway, New Jersey, 08854, USA \\  \\
  E-mail: \{\texttt{xtyuan1980@gmail.com}, \texttt{pingli@stat.rutgers.edu}, \texttt{tzhang@stat.rutgers.edu}\}
  }
\date{}

\begin{document}

\maketitle

\begin{abstract}
Hard Thresholding Pursuit (HTP) is an iterative greedy selection
procedure for finding sparse solutions of underdetermined linear
systems. This method has been shown to have strong theoretical
guarantee and impressive numerical performance. In this paper, we
generalize HTP from compressive sensing to a generic problem setup
of sparsity-constrained convex optimization. The proposed algorithm
iterates between a standard gradient descent step and a hard
thresholding step with or without debiasing. We prove that our
method enjoys the strong guarantees analogous to HTP in terms of
rate of convergence and parameter estimation accuracy. Numerical
evidences show that our method is superior to the state-of-the-art
greedy selection methods in sparse logistic regression and sparse
precision matrix estimation tasks.
\end{abstract}


%\subparagraph{AMS subject classifications (2010).} Provide up to
%five subject classification codes here; search for the string
%``MSC'' at \url"www.ams.org".

\newpage

\section{Introduction}

In the past decade, high-dimensional data analysis has received
broad research interests in data mining and scientific discovery,
with many significant results obtained in theory, algorithm and
applications. The major driven force is the rapid development of
data collection technologies in many applications domains such as social
networks, natural language processing, bioinformatics and
computer vision. In these applications it is not unusual that data
samples are represented with millions or even billions of features
using which an underlying statistical learning model must be
fit. In many circumstances, however, the number of collected samples
is substantially smaller than the dimensionality of the feature,
implying that consistent estimators cannot be hoped for unless
additional assumptions are imposed on the model. One of the widely
acknowledged prior assumptions is that the data exhibit low-dimensional structure, which can
often be captured by imposing sparsity constraint on the model
parameter space. It is thus crucial to develop robust and
efficient computational procedures for solving, even just approximately,
these optimization problems with sparsity constraint.

In this paper, we focus on the following generic sparsity-constrained
optimization problem
\begin{equation}\label{prob:f_l0}
\min_{x \in \mathbb{R}^p} f(x), \quad \st \|x\|_0 \le k,
\end{equation}
where $f:\mathbb{R}^p \mapsto \mathbb{R}$ is a smooth convex cost
function. Among others, several examples falling into this model
include: (i) Sparsity-constrained linear regression
model~\citep{Tropp-2007} where the residual error is used to measure
data reconstruction error; (ii) Sparsity-constrained logistic
regression model~\citep{Bahmani-2013} where the sigmoid loss is used
to measure prediction error; (iii) Sparsity-constrained graphical
model learning~\citep{Jalali-NIPS-2011} where the likelihood of
samples drawn from an underlying probabilistic model is used to
measure data fidelity.

Unfortunately, due to the non-convex cardinality constraint, the
problem~\eqref{prob:f_l0} is generally NP-hard even for the
quadratic cost function~\citep{Natarajan-SJCom-1995}. Thus, one must
instead seek approximate solutions. In particular, the special case
of~\eqref{prob:f_l0} in least square regression models has gained
significant attention in the area of compressed
sensing~\citep{Donoho-TIT-2006}. A vast body of greedy selection algorithms for compressing sensing have
been proposed including matching pursuit~\citep{Mallat-1993},
orthogonal matching pursuit~\citep{Tropp-2007}, compressive sampling
matching pursuit~\citep{Needell-2009}, hard thresholding
pursuit~\citep{Foucart-2011}, iterative hard thresholding~\citep{Blumensath-2009} and subspace
pursuit~\citep{Dai-TIT-2009} to name a few. These algorithms successively select the position of nonzero
entries and estimate their values via exploring the residual error from the previous
iteration. Comparing to those first-order convex
optimization methods developed for $\ell_1$-regularized sparse
learning~\citep{Beck-2009,Agarwal-NIPS-2010}, these greedy selection algorithms often exhibit similar accuracy guarantees but more attractive computational efficiency.

The least square error used in compressive sensing, however, is not an appropriate measure of discrepancy in a variety of applications beyond signal processing. For example, in statistical machine learning the
log-likelihood function is commonly used in logistic regression
problems~\citep{Bishop-PRML-2006} and graphical models
learning~\citep{Jalali-NIPS-2011,Ravikumar-EJS-2011}. Thus, it is
desirable to investigate theory and algorithms applicable to a broader
class of sparsity-constrained learning problems as given
in~\eqref{prob:f_l0}. To this end, several forward selection
algorithms have been proposed to select the nonzero entries in a
sequential fashion~\citep{Kim-ICML-2004,Shai-Shalev-Shwartz-SIJO-2010,Yuan-2013,Jaggi-2011-phd}.
This category of methods date back to the Frank-Wolfe
method~\citep{Frank-Wolfe-1956}. The forward greedy selection method
has also been generalized to minimize a convex objective over the
linear hull of a collection of
atoms~\citep{Tewari-NIPS-2011,Yuan-2012}. To make the greedy
selection procedure more adaptive, \citet{Zhang-FoBa-2008} proposed
a forward-backward algorithm which takes backward steps adaptively
whenever beneficial. \citet{Jalali-NIPS-2011} have applied this
forward-backward selection method to learn the structure of a
sparse graphical model. More recently, \citet{Bahmani-2013} proposed
a gradient hard-thresholding method which generalizes the
compressive sampling matching pursuit method~\citep{Needell-2009}
from compressive sensing to the general sparsity-constrained
optimization problem. The hard-threshholding-type methods have
also been shown to be statistically and computationally efficient
for sparse principal component analysis~\citep{Yuan-JMLR-2013,Ma-2013}.

\subsection{Our Contribution}

In this paper, inspired by the success of Hard Thresholding Pursuit
(HTP)~\citep{Foucart-2011,Foucart-2012} in compressive sensing, we
propose the Gradient Hard Thresholding Pursuit (GraHTP) method to
encompass the sparse estimation problems arising from applications
with general nonlinear models. At each iteration, GraHTP performs
standard gradient descent followed by a hard thresholding operation
which first selects the top $k$ (in magnitude) entries of the
resultant vector and then (optionally) conducts debiasing on the
selected entries. We prove that under mild conditions GraHTP (with
or without debiasing) has strong theoretical guarantees analogous to
HTP in terms of convergence rate and parameter estimation accuracy.
We have applied GraHTP to the sparse logistic regression model and
the sparse precision matrix estimation model, verifying that the
guarantees of HTP are valid for these two models. Empirically we
demonstrate that GraHTP is comparable or superior to the
state-of-the-art greedy selection methods in these two sparse
learning models. To our knowledge, GraHTP is the first
gradient-descent-truncation-type method for sparsity constrained
nonlinear problems.

\subsection{Notation}

In the following, $x \in \mathbb{R}^p$ is a vector, $F$ is an index
set and $A$ is a matrix. The following notations will be used in the
text.
\begin{itemize}
\item $[x]_i$: the $i$th entry of vector $x$.
\item $x_F$: the restriction of $x$ to index set $F$, i.e., $[x_F]_{i}=[x]_i$ if $i \in F$, and $[x_F]_{i}=0$
otherwise.
\item $x_k$: the restriction of $x$ to the top $k$ (in modulus)
entries. We will simplify $x_{F_k}$ to $x_k$ without ambiguity in
the context.
\item $\|x\|=\sqrt{x^\top x}$: the Euclidean norm of $x$.
\item $\|x\|_1 =\sum_{i=1}^d |x_i|$: the $\ell_1$-norm of $x$.
\item $\|x\|_0$: the number of nonzero entries of $x$.
\item $\supp(x)$: the index set of nonzero entries of $x$.
\item $\supp(x,k)$: the index set of the top $k$ (in modulus) entries of $x$.
\item $[A]_{ij}$: the element on the $i$th row and $j$th column of  matrix $A$.
\item $\|A\|=\sup_{\|x\| \le 1} \|Ax\|$: the spectral norm of matrix $A$.
\item $A_{F\bullet}$ ($A_{\bullet F}$ ): the rows (columns) of matrix $A$ indexed in $F$.
\item $|A|_1 = \sum_{1 \le i \le p, 1 \le j \le q}|[A]_{ij}|$: the element-wise $\ell_1$-norm of $A$.
\item $\Tr(A)$: the trace (sum of diagonal elements) of a square matrix $A$.
\item $A^-$: the restriction of a square matrix $A$ on its off-diagonal entries
\item $\vect(A)$: (column wise) vectorization of a matrix $A$.
\end{itemize}

\subsection{Paper Organization}

This paper proceeds as follows: We present in~\S\ref{sect:GraHTP}
the GraHTP algorithm. The convergence guarantees of GraHTP are
provided in~\S\ref{sect:theory}. The specializations of GraHTP in
logistic regression and Gaussian graphical models learning are
investigated in~\S\ref{sect:applications}. Monte-Carlo simulations
and experimental results on real data are presented
in~\S\ref{sect:experiments}. We conclude this paper
in~\S\ref{sect:conclusion}.

\section{Gradient Hard Thresholding Pursuit}
\label{sect:GraHTP}

GraHTP is an iterative greedy selection procedure for approximately
optimizing the non-convex problem~\eqref{prob:f_l0}. A high level
summary of GraHTP is described in the top panel of
Algorithm~\ref{alg:GraHTP}. The procedure generates a sequence of
intermediate $k$-sparse vectors $x^{(0)}, x^{(1)}, \ldots$ from an
initial sparse approximation $x^{(0)}$ (typically $x^{(0)}=0$). At
the $t$-th iteration, the first step \textbf{S1}, $\tilde x^{(t)} = x^{(t-1)}
- \eta \nabla f(x^{(t-1)})$, computes the gradient descent at the
point $x^{(t-1)}$ with step-size $\eta$. Then in the second step,
\textbf{S2}, the $k$ coordinates of the vector $\tilde x^{(t)}$ that have the
largest magnitude are chosen as the support in which pursuing the
minimization will be most effective. In the third step, \textbf{S3}, we find
a vector with this support which minimizes the objective function,
which becomes $x^{(t)}$. This last step, often referred to as
\emph{debiasing}, has been shown to improve the performance in other
algorithms too. The iterations continue until the algorithm reaches
a terminating condition, e.g., on the change of the cost function or
the change of the estimated minimum from the previous iteration. A
natural criterion here is $F^{(t)} = F^{(t-1)}$ (see \textbf{S2} for the definition of $F^{(t)}$), since then
$x^{(\tau)} = x^{(t)}$ for all $\tau \ge t$, although there is no
guarantee that this should occur. It will be assumed throughout the
paper that the cardinality $k$ is known. In practice this quantity
may be regarded as a tuning parameter of the algorithm via, for example,
cross-validations.

In the standard form of GraHTP, the debiasing step \textbf{S3} requires to
minimize $f(x)$ over the support $F^{(t)}$. If this step is judged
too costly, we may consider instead a fast variant of GraHTP, where
the debiasing is replaced by a simple truncation operation $x^{(t)}
= \tilde x^{(t)}_k$. This leads to the Fast GraHTP (FGraHTP)
described in the bottom panel of Algorithm~\ref{alg:GraHTP}. It is
interesting to note that FGraHTP can be regarded as a projected
gradient descent procedure for optimizing the non-convex
problem~\eqref{prob:f_l0}. Its per-iteration computational overload
is almost identical to that of the standard gradient descent
procedure. While in this paper we only study the FGraHTP outlined
in Algorithm 1, we should mention that other fast variants of GraHTP can also be
considered. For instance, to reduce the computational cost of \textbf{S3},
we can take a restricted Newton step or a restricted gradient
descent step to calculate $x^{(t)}$.

We close this section by pointing out that, in the special case where
the squared error $f(x) = \frac{1}{2}\|y - Ax\|^2$ is the cost
function, GraHTP reduces to HTP~\citep{Foucart-2011}. Specifically,
the gradient descent step S1 reduces to $\tilde x^{(t)} =
x^{(t-1)}+\eta A^\top (y - Ax^{(t-1)})$ and the debiasing step S3
reduces to the orthogonal projection $x^{(t)} = \argmin\{ \|y-Ax\|,
\supp(x) \subseteq F^{(t)}\}$. In the meanwhile, FGraHTP reduces to
IHT~\citep{Blumensath-2009} in which the iteration is defined by
$x^{(t)} = (x^{(t-1)} + \eta A^\top (y-Ax^{(t-1)}))_k$.

\begin{algorithm} \label{alg:GraHTP}
\SetKwInOut{Input}{Input}\SetKwInOut{Output}{Output}\SetKw{Initialization}{Initialization:}

\Initialization{$x^{(0)}$ with $\|x^{(0)}\|_0\le k$ (typically
$x^{(0)}=0$), $t=1$.}

\Output{$x^{(t)}$.}

\Repeat{halting condition holds} {

(\textbf{S1}) Compute $\tilde x^{(t)} = x^{(t-1)} - \eta \nabla
f(x^{(t-1)})$;

(\textbf{S2}) Let $F^{(t)} = \supp(\tilde x^{(t)},k)$ be the indices of
$\tilde x^{(t)}$ with the largest $k$ absolute values;

(\textbf{S3}) Compute $x^{(t)} = \argmin\{ f(x), \supp(x) \subseteq
F^{(t)}\}$;

$t=t+1$;

}

-----------------------------------------------$\bigstar$ \emph{Fast GraHTP} $\bigstar$------------------------------------------------

\Repeat{halting condition holds} {

Compute $\tilde x^{(t)} = x^{(t-1)} - \eta \nabla f(x^{(t-1)})$;

Compute $x^{(t)} = \tilde x^{(t)}_k$ as the truncation of $ \tilde
x^{(t)}$ with top $k$ entries preserved;

$t=t+1$;

}
 \caption{Gradient Hard Thresholding Pursuit (GraHTP).}
\end{algorithm}

\section{Theoretical Analysis}
\label{sect:theory}

In this section, we analyze the theoretical properties of GraHTP and
FGraHTP. We first study the convergence of these two algorithms.
Next, we investigate their performances for the task of sparse recovery  in terms of
convergence rate and parameter estimation accuracy. We require the
following key technical condition under which the convergence and
parameter estimation accuracy of GraHTP/FGraHTP can be guaranteed.
To simplify the notation in the following analysis, we abbreviate
$\nabla_F f = (\nabla f)_F$ and $\nabla_s f = (\nabla f)_{s}$.

\begin{definition}[Condition $C(s,\zeta,\rho_s)$]\label{assump_1}
For any integer $s>0$, we say $f$ satisfies condition $C(s,\zeta,\rho_s)$ if for any index set $F$ with cardinality $|F|\le s$ and any $x,y$
with $\supp(x)\cup \supp(y)\subseteq F$, the following inequality
holds for some $\zeta>0$ and $0<\rho_s<1$:
\[
\|x - y - \zeta \nabla_F f(x) + \zeta \nabla_F f(y)\| \le
\rho_{s}\|x-y\|.
\]
\end{definition}
\begin{remark}
In the special case where $f(x)$ is least square loss function and
$\zeta=1$, Condition $C(s,\zeta,\rho_s)$ reduces to the well known
\emph{Restricted Isometry Property } (RIP) condition in compressive sensing.
\end{remark}
We may establish the connections between condition $C(s,\zeta,\rho_s)$ and the conditions of restricted strong convexity/smoothness which
are key to the analysis of several previous greedy selection
methods~~\citep{Zhang-FoBa-2008,Shai-Shalev-Shwartz-SIJO-2010,Yuan-2013,Bahmani-2013}.
\begin{definition}[Restricted Strong Convexity/Smoothness]
For any integer $s>0$, we say $f(x)$ is restricted $m_s$-strongly
convex and $M_s$-strongly smooth if there exist $\exists m_s, M_s >
0$ such that
\begin{equation}\label{inequat:strong_smooth_convex}
\frac{m_s}{2}\|x-y\|^2 \le f(x) - f(y) - \langle \nabla f(y),
x-y\rangle \le \frac{M_s}{2}\|x-y\|^2, \quad \forall\|x-y\|_0\le s.
\end{equation}
\end{definition}
The following lemma connects condition $C(s,\zeta,\rho_s)$ to the
restricted strong convexity/smoothness conditions.
\begin{lemma}\label{lemma:strong_smooth}
Assume that $f$ is a differentiable function.
\begin{itemize}
  \item[(a)] If $f$ satisfies condition $C(s,\zeta,\rho_s)$, then for all $\|x-y\|_0 \le s$ the following two inequalities hold:
 \begin{eqnarray}
  \frac{1-\rho_s}{\zeta}\|x-y\| &\le& \|\nabla_F f(x) - \nabla_F f(y)\|\le
\frac{1+\rho_s}{\zeta} \|x-y\| \nonumber \\
 f(x) &\le& f(y) + \langle \nabla f(y), x-y\rangle +
\frac{1+\rho_s}{2\zeta}\|x-y\|^2. \nonumber
\end{eqnarray}
  \item[(b)] If $f$ is $m_s$-strongly convex and $M_s$-strongly
  smooth, then $f$ satisfies condition $C(s,\zeta,\rho_s)$ with any
  \[
  \zeta < 2m_s/M^2_s , \quad \ \rho_s = \sqrt{1-2\zeta m_s +
  \zeta^2 M_s^2}.
  \]
\end{itemize}
\end{lemma}
A proof of this lemma is provided in
Appendix~\ref{append:proof_lemma_strong_smooth}.
\begin{remark}
The Part(a) of Lemma~\ref{lemma:strong_smooth} indicates that if
condition $C(s,\zeta,\rho_s)$ holds, then $f$ is strongly smooth. The
Part(b) of Lemma~\ref{lemma:strong_smooth} shows that the strong
smoothness/convexity conditions imply condition $C(s,\zeta,\rho_s)$.
Therefore, condition $C(s,\zeta,\rho_s)$ is no stronger than the strong
smoothness/conveixy conditions.
\end{remark}

\subsection{Convergence}

We now analyze the convergence properties of GraHTP and FGraHTP.
First and foremost, we make a simple observation about GraHTP: since
there is only a finite number of subsets of $\{1,...,p\}$ of size
$k$, the sequence defined by GraHTP is eventually periodic. The
importance of this observation lies in the fact that, as soon as the
convergence of GraHTP is established, then we can certify that the
limit is exactly achieved after a finite number of iterations. We
establish in Theorem~\ref{thrm:convergence}  the convergence of GraHTP and FGraHTP under proper
conditions. A proof of this theorem is provided in
Appendix~\ref{append:proof_thrm_convergence}.

\begin{theorem}\label{thrm:convergence}
Assume that $f$ satisfies condition $C(2k, \zeta, \rho_{2k})$ and the step-size $\eta
< \zeta / (1+\rho_{2k})$. Then the sequence $\{x^{(t)}\}$ defined by
GraHTP converges in a finite number of iterations. Moreover, the
sequence $\{f(x^{(t)})\}$ defined by FGraHTP converges.
\end{theorem}

\begin{remark}
Since $\rho_{2k}\in (0,1)$, we have that the convergence results in
Theorem~\ref{thrm:convergence} hold whenever the step-size $\eta <
\zeta /2$. If $f$ is $m_s$-strongly convex and $M_s$-strongly
smooth, then from Part(b) of Lemma~\ref{lemma:strong_smooth} we know
that Theorem~\ref{thrm:convergence} holds whenever the step-size
$\eta < m_s/M_s^2$.
\end{remark}

\subsection{Sparse Recovery Performance}

The following theorem is our main result on the parameter estimation
accuracy of GraHTP and FGraHTP when the target solution is sparse.

\begin{theorem}\label{thrm:recovery_GraHTP}
Let $\bar x$ be an arbitrary $\bar k$-sparse vector and $k \ge \bar
k$. Let $s=2k+\bar k$. If $f$ satisfies condition $C(s,\zeta,\rho_s)$ and $\eta <
\zeta$,
\begin{itemize}
  \item[(a)] if $\mu_1 =
\sqrt{2}(1-\eta/\zeta + (2-\eta/\zeta)\rho_s)/(1-\rho_s )<1$, then
at iteration $t$, GraHTP will recover an approximation $x^{(t)}$
satisfying
\begin{equation}\label{equat:thrm_gradient_case_0}
\|x^{(t)} - \bar x\| \leq \mu_1^t \|x^{(0)} - \bar x\| +
\frac{2\eta+\zeta}{(1-\mu_1)(1-\rho_s)}\|\nabla_k f(\bar x)\|.
\nonumber
\end{equation}
  \item[(b)] if $\mu_2 =
2(1-\eta/\zeta + (2-\eta/\zeta)\rho_s)<1$, then at iteration $t$,
FGraHTP will recover an approximation $x^{(t)}$ satisfying
\begin{equation}\label{equat:thrm_gradient_case_0}
\|x^{(t)} - \bar x\| \leq \mu_2^t \|x^{(0)} - \bar x\| +
\frac{2\eta}{1-\mu_2}\|\nabla_s f(\bar x)\|. \nonumber
\end{equation}
\end{itemize}
\end{theorem}
A proof of this theorem is provided in
Appendix~\ref{append:proof_thrm_recovery_grahtp}. Note that we did
not make any attempt to optimize the constants in
Theorem~\ref{thrm:recovery_GraHTP}, which are relatively loose.
In the discussion, we  ignore the constants and
focus on the main message Theorem~\ref{thrm:recovery_GraHTP}
conveys. The part (a) of Theorem~\ref{thrm:recovery_GraHTP}
indicates that under proper conditions, the estimation error of
GraHTP is determined by the multiple of $\|\nabla_k f(\bar x)\|$,
and the rate of convergence before reaching this error level is
geometric. Particularly, if the sparse vector $\bar x$ is
sufficiently close to an unconstrained minimum of $f$ then the
estimation error floor is negligible because $\nabla_k f(\bar x)$
has small magnitude. In the ideal case where $\nabla f(\bar x)=0$
(i.e., the sparse vector $\bar x$ is an unconstrained minimum of
$f$), this result guarantees that we can recover $\bar x$ to
arbitrary precision. In this case, if we further assume that $\eta$
satisfies the conditions in Theorem~\ref{thrm:convergence}, then
exact recovery can be achieved in a \emph{finite} number of iterations. The part
(b) of Theorem~\ref{thrm:recovery_GraHTP} shows that FGraHTP enjoys
a similar geometric rate of convergence and the estimation error is
determined by the multiple of $\|\nabla_s f(\bar x)\|$ with
$s=2k+\bar k$.

The shrinkage rates $\mu_1 < 1$ (see Part (a)) and $\mu_2<1$ (see
Part (b)) respectively control the convergence rate of GraHTP and
FGraHTP. For GraHTP, the condition $\mu_1<1$ implies
\begin{equation}\label{inequat:eta}
\eta > \frac{((2\sqrt{2}+1)\rho_s + \sqrt{2}-1)\zeta}{\sqrt{2} +
\sqrt{2}\rho_s}.
\end{equation}
By combining this condition with $\eta < \zeta$, we can see that $\rho_s
< 1/(\sqrt{2}+1)$ is a necessary condition to guarantee $\mu_1<1$.
On the other side, if $\rho_s < 1/(\sqrt{2}+1)$, then we can always
find a step-size $\eta<\zeta$ satisfying~\eqref{inequat:eta} such
that $\mu_1<1$. This condition of $\rho_s$ is analogous to
the RIP condition for estimation from noisy measurements in
compressive sensing~\citep{Candes-CPAM-2006,Needell-2009,Foucart-2011}. Indeed, in this setup our GraHTP algorithm reduces to HTP which requires weaker RIP condition than prior compressive sensing algorithms. The guarantees of GraHTP and HTP are almost identical, although we did not make any attempt to optimize the RIP sufficient constants, which are $1/(\sqrt{2}+1)$ (for GraHTP) versus $1/\sqrt{3}$ (for HTP). We would like to emphasize that the condition $\rho_s<1/(\sqrt{2}+1)$ derived for GraHTP also holds in fairly general setups beyond compressive sensing. For FGraHTP we have similar discussions.

For the
general sparsity-constrained optimization problem, we note that a
similar estimation error bound has been established for the GraSP
(Gradient Support Pursuit) method~\citep{Bahmani-2013} which is another hard-thresholding-type method. At time stamp $t$, GraSP first conducts debiasing over the union of the top $k$ entries of $x^{(t-1)}$ and the top $2k$ entries of $\nabla f(x^{(t-1)})$, then it selects the top $k$ entries of the resultant vector and updates their values via debiasing, which becomes $x^{(t)}$. Our GraHTP is connected to GraSP in the sense that the $k$ largest absolute elements after the gradient descent step (see \textbf{S1} and \textbf{S2} of Algorithm~\ref{alg:GraHTP}) will come from some combination of the largest elements in $x^{(t-1)}$ and the largest elements in the gradient $\nabla f(x^{(t-1)})$. Although the convergence rate are of the same order, the per-iteration cost of GraHTP is cheaper than GraSP. Indeed, at each iteration, GraSP needs to minimize the objective over a support of size $3k$ while that size for GraHTP is $k$. FGraHTP is even cheaper for iteration as it does not need any debiasing operation. We will compare the actual numerical performances of these methods in our empirical study.

\section{Applications}
\label{sect:applications}

In this section, we will specialize GraHTP/FGraHTP to two popular
statistical learning models: the sparse logistic regression (in
\S\ref{ssect:logistic_model}) and the sparse precision matrix
estimation (in \S\ref{ssect:ggm}).

\subsection{Sparsity-Constrained $\ell_2$-Regularized Logistic Regression}
\label{ssect:logistic_model}

Logistic regression is one of the most popular models in
statistics and machine
learning~\citep{Bishop-PRML-2006}. In this model
the relation between the random feature vector $u \in \mathbb{R}^p$
and its associated random binary label $v \in \{-1,+1\}$ is
determined by the conditional probability
\begin{equation}\label{equat:logistic}
\mathbb{P}(v|u; \bar w) = \frac{\exp (2v \bar w^\top u)}{1+\exp (2v
\bar w^\top u)},
\end{equation}
where $\bar w \in \mathbb{R}^p$ denotes a parameter vector. Given a
set of $n$ independently drawn data samples
$\{(u^{(i)},v^{(i)})\}_{i=1}^n$, logistic regression learns the
parameters $w $ so as to minimize the logistic log-likelihood given
by
\[
l(w): = -\frac{1}{n}\log \prod_i \mathbb{P}( u^{(i)} \mid v^{(i)};
w) = \frac{1}{n}\sum_{i=1}^n \log(1+\exp(-2v^{(i)} w^\top u^{(i)})).
\]
It is well-known that $l(w)$ is convex. Unfortunately, in high-dimensional
setting, i.e., $n<p$, the problem can be underdetermined and thus its minimum is not unique. A conventional way to handle this issue is to impose
$\ell_2$-regularization to the logistic loss to avoid singularity. The $\ell_2$-penalty, however, does not promote sparse
solutions which are often desirable in high-dimensional learning tasks. The
sparsity-constrained $\ell_2$-regularized logistic regression is
then given by:
\begin{equation}\label{prob:logisic}
\min_w f(w) = l(w) + \frac{\lambda}{2} \|w\|^2 , \quad \text{subject
to } \|w\|_0 \le k,
\end{equation}
where $\lambda>0$ is the regularization strength parameter. Obviously
$f(w)$ is $\lambda$-strongly convex and hence it has a unique minimum.
The cardinality constraint enforces the solution to be sparse.
%It has been shown in~\citep[Corollary 1]{Bahmani-2013} that with
%overwhelming probability $f(w)$ is strong strongly smooth and
%strongly convex. Thus, from
%Remark~\ref{remmark:strong_smooth_convex} we know that $f(w)$
%satisfies Assumption~\ref{assump_1} with overwhelming probability.
%The bounds of $\|\nabla_k f(w)\|_2$ for the $\ell_2$-regularized
%logistic loss have also been analyzed in~\citep{Bahmani-2013}.

\subsubsection{Verifying Condition $C(s,\zeta,\rho_s)$}

Let $U=[u^{(1)},...,u^{(n)}]\in\mathbb{R}^{p\times n}$ be the design matrix and $\sigma(z)=1/(1+\exp(-z))$ be the sigmoid function. In the case of $\ell_2$-regularized logistic loss considered in this
section we have
\[
\nabla f(w) = U a(w)/n + \lambda w,
\]
where the vector $a(w)\in \mathbb{R}^n$ is given by
$[a(w)]_i=-2v^{(i)}(1-\sigma(2v^{(i)}w^\top u^{(i)}))$. The
following result verifies that $f(w)$ satisfies
Condition $C(s,\zeta,\rho_s)$ under mild conditions.
\begin{proposition}\label{prop:logistic_assump1}
Assume that for any index set $F$ with $|F| \le s$ we have $\forall
i$, $\|(u^{(i)})_F\| \le R_s$. Then the $\ell_2$-regularized
logistic loss satisfies Condition $C(s,\zeta,\rho_s)$ with any
\[
\zeta < \frac{2\lambda}{(4\sqrt{s}R_s^2+\lambda)^2} , \quad
\rho_s=\sqrt{1- 2 \zeta \lambda + \zeta^2
(4\sqrt{s}R_s^2+\lambda)^2}.
\]
\end{proposition}
A proof of this result is given in
Appendix~\ref{append:proof_prop_logistic_assump1}.

\subsubsection{Bounding the Estimation Error}

We are going to bound $\|\nabla_s f(\bar w)\|$ which we obtain from
Theorem~\ref{thrm:recovery_GraHTP} that controls the estimation
error bounds of GraHTP (with $s=k$) and FGraHTP (with $s=2k+\bar
k$). In the following deviation, we assume that the joint density of
the random vector $(u,v) \in \mathbb{R}^{p+1}$ is given by the
following exponential family distribution:
\begin{equation}\label{equat:joint_distr}
\mathbb{P}(u,v; \bar w) = \exp\left(v \bar w^\top u + B(u) - A(\bar
w) \right),
\end{equation}
where
\[
A(\bar w):= \log \sum_{v=\{-1,1\}}\int_{\mathbb{R}^p} \exp\left(v
\bar w^\top u + B(u) \right) du
\]
is the log-partition function. The term $B(u)$ characterizes the
marginal behavior of $u$. Obviously, the conditional distribution of
$v$ given $u$, $\mathbb{P}(v \mid u; \bar w)$, is given by the
logistical model~\eqref{equat:logistic}. By trivial algebra we can
obtain the following standard result which shows that the first
derivative of the logistic log-likelihood $l(w)$ yields the
cumulants of the random variables $v[u]_j$~\citep[see,
e.g.,][]{Wainwright-2008}:
\begin{equation}\label{equat:derivatives}
\frac{\partial l}{\partial [w]_j} = \frac{1}{n}\sum_{i=1}^n
\left\{-v^{(i)} [u^{(i)}]_j + \mathbb{E}_{v}[v[u^{(i)}]_j\mid
u^{(i)}]\right\}.
\end{equation}
Here the expectation $\mathbb{E}_{v}[\cdot \mid u]$ is taken over
the conditional distribution~\eqref{equat:logistic}. We introduce
the following sub-Gaussian condition on the random variate $v[u]_j$.
\begin{assumption}\label{assump:tail_1}
For all $j$, we assume that there exists constant $\sigma>0$ such
that for all $\eta$,
\[
\mathbb{E}[\exp(\eta v[u]_j)] \le \exp\left(\sigma^2\eta^2/2\right).
\]
\end{assumption}
This assumption holds when $[u]_j$ are sub-Gaussian (e.g., Gaussian
or bounded) random variables. The following result establishes the bound of
$\|\nabla_s f(\bar w)\|$.
\begin{proposition}\label{prop_logsitic}
If Assumption~\ref{assump:tail_1} holds, then with probability at
least $1-4p^{-1}$,
\[
\|\nabla_s f(\bar w)\| \le 4\sigma\sqrt{s\ln p/n} + \lambda \|\bar
w_s\|.
\]
\end{proposition}
A proof of this result can be found in
Appendix~\ref{append:proof_prop_logsitic}.
\begin{remark}
If we choose $\lambda = O(\sqrt{\ln p/ n})$, then with overwhelming
probability $\|\nabla_s f(\bar w)\|$ vanishes at the rate of $
O(\sqrt{s\ln p /n})$. This bound is superior to the bound
provided by~\citet[Section 4.2]{Bahmani-2013} which is
non-vanishing.
\end{remark}

\subsection{Sparsity-Constrained Precision Matrix Estimation}
\label{ssect:ggm}

An important class of sparse learning problems involves estimating
the precision (inverse covariance) matrix of high dimensional random
vectors under the assumption that the true precision matrix is
sparse. This problem arises in a variety of applications, among them
computational biology, natural language processing and document
analysis, where the model dimension may be comparable or
substantially larger than the sample size. 

Let $x$ be a $p$-variate random vector with zero-mean Gaussian
distribution $\mathcal {N}(0, \bar \Sigma)$. Its density is
parameterized by the precision matrix $\bar \Omega =
\bar\Sigma^{-1}\succ 0$ as follows:
\[
\phi(x; \bar\Omega) = \frac{1}{\sqrt{(2\pi)^p (\det
\bar\Omega)^{-1}}} \exp\left( -\frac{1}{2}x^\top \bar\Omega x
\right).
\]
It is well known that the conditional independence between the variables $[x]_i$ and $[x]_j$ given $\{[x]_k, k \neq i, j\}$ is
equivalent to $[\bar \Omega]_{ij} = 0$. The conditional independence relations between components of $x$, on the other hand, can be represented by a graph $\mathcal {G} = (V,E)$ in which the vertex set $V$
has $p$ elements corresponding to $[x]_1,...,[x]_p$, and the edge
set $E$ consists of edges between node pairs $\{[x]_i,[x]_j\}$. The edge between
$[x]_i$ and $[x]_j$ is excluded from $E$ if and only if $[x]_i$ and
$[x]_j$ are conditionally independent given other variables. This graph is known as Gaussian Markov random field
(GMRF)~\citep{Edwards-2000}. Thus for multivariate Gaussian distribution, estimating the support of the precision matrix $\bar
\Omega$ is equivalent to learning the structure of GMRF $\mathcal {G}$.  

Given i.i.d. samples $\mathbb{X}_n = \{x^{(i)}\}_{i=1}^n$ drawn from
$\mathcal {N}(0, \bar\Sigma)$, the negative log-likelihood, up to a
constant, can be written in terms of the precision matrix as
\[
\mathcal {L}(\mathbb{X}_n;\bar\Omega):= -\log\det \bar\Omega +
\langle \Sigma_n, \bar\Omega \rangle,
\]
where $\Sigma_n$ is the sample covariance matrix. We are interested
in the problem of estimating a sparse precision $\bar \Omega$ with
no more than a pre-specified number of off-diagonal non-zero
entries. For this purpose, we consider the following cardinality
constrained log-determinant program:
\begin{equation}\label{prob:SISC_card_constraint}
\min_{\Omega \succ 0} L(\Omega):=-\log\det \Omega + \langle
\Sigma_n, \Omega \rangle, \quad \st \|\Omega^-\|_0 \le 2k,
\end{equation}
where $\Omega^{-}$ is the restriction of $\Omega$ on the
off-diagonal entries, $\|\Omega^{-}\|_0=|\supp(\Omega^{-})|$ is the
cardinality of the support set of $\Omega^{-}$ and integer $k>0$
bounds the number of edges $|E|$ in GMRF.

\subsubsection{Verifying Condition $C(s,\zeta,\rho_s)$}

It is easy to show that the Hessian $\nabla^2 L(\Omega) =
\Omega^{-1} \otimes \Omega^{-1}$, where $\otimes$ is the Kronecker
product operator. The following result shows that $L$ satisfies
Condition $C(s,\zeta,\rho_s)$ if the eigenvalues of $\Omega$ are lower
bounded from zero and upper bounded.
\begin{proposition}\label{prop:logdet_SRH}
Suppose that $\|\Omega^-\|_0 \le s$ and $\alpha_s I \preceq \Omega
\preceq \beta_s I$ for some $0 < \alpha_s \le \beta_s$. Then
$L(\Omega)$ satisfies Condition $C(s,\zeta,\rho_s)$ with any
\[
\zeta < \frac{2\alpha_s^4}{\beta_s^2} , \quad \rho_s=\sqrt{1- 2
\zeta \beta_s^{-2} + \zeta^2 \alpha_s^{-4}}.
\]
\end{proposition}
\begin{proof}
Due to the fact that the eigenvalues of Kronecker products of
symmetric matrices are the products of the eigenvalues of their
factors, it holds that
\[
\beta_s^{-2} I \preceq \Omega^{-1}\otimes \Omega^{-1} \preceq
\alpha_s^{-2} I.
\]
Therefore we have $\beta_s^{-2} \le \|\nabla ^2 L(\Omega)\|_2 \le
\alpha_s^{-2}$ which implies that $L(\Omega)$ is
$\beta_s^{-2}$-strongly convex and $\alpha_s^{-2}$-strongly smooth.
The desired result follows directly from the Part(b) of
Lemma~\ref{lemma:strong_smooth}.
\end{proof}
Motivated by Proposition~\ref{prop:logdet_SRH}, we consider applying
GraHTP to the following modified version of
problem~\eqref{prob:SISC_card_constraint}:
\begin{equation}\label{prob:SISC_card_constraint_bounds}
\min_{\alpha I \preceq \Omega \preceq \beta I} L(\Omega), \quad \st
\|\Omega^-\|_0 \le 2k,
\end{equation}
where $0<\alpha \le \beta$ are two constants which respectively
lower and upper bound the eigenvalues of the desired solution. To
roughly estimate $\alpha$ and $\beta$, we employ a rule proposed
by~\citet[Proposition 3.1]{Lu-VSM-2009} for the $\ell_1$
log-determinant program. Specifically, we set
\[
\alpha = (\|\Sigma_n\|_2 + n\xi)^{-1}, \quad \beta =\xi^{-1}(n -
\alpha\Tr(\Sigma_n)),
\]
where $\xi$ is a small enough positive number (e.g., $\xi=10^{-2}$
as utilized in our experiments).

\subsubsection{Bounding the Estimation Error.}

Let $h:= \vect(\nabla L(\bar\Omega))$. It is known from
Theorem~\ref{thrm:recovery_GraHTP} that the estimation error is
controlled by $\|h_s\|_2$. Since $\nabla L(\bar\Omega) =
-\bar\Omega^{-1} + \Sigma_n = \Sigma_n - \bar\Sigma$, we have
$\|h_s\|_2 \le \sqrt{s} |\Sigma_n - \bar\Sigma|_\infty$. It is known
that $|\Sigma_n - \bar\Sigma|_\infty \le \sqrt{\log p/n}$ with
probability at least $1-c_0p^{-c_1}$ for some positive constants
$c_0$ and $c_1$ and sufficiently large $n$~\citep[see, e.g.,][Lemma
1]{Ravikumar-EJS-2011}. Therefore with overwhelming probability we
have $\|h_s\|_2 = O(\sqrt{s \log p/n})$ when $n$ is sufficiently
large.

\subsubsection{A Modified GraHTP}

Unfortunately, GraHTP is not directly applicable to the
problem~\eqref{prob:SISC_card_constraint_bounds} due to the presence
of the constraint $\alpha I \preceq \Omega \preceq \beta I$ in
addition to the sparsity constraint. To address this issue, we need
to modify the debiasing step (S3) of GraHTP to minimize $L(\Omega)$
over the constraint of $\alpha I \preceq \Omega \preceq \beta I$ as
well as the support set $F^{(t)}$:
\begin{equation}\label{subprob:omega_F}
\min_{\alpha I \preceq \Omega \preceq \beta I} L(\Omega), \quad \st
\supp(\Omega) \subseteq F^{(t)}.
\end{equation}
Since this problem is convex, any off-the-shelf convex solver can be
applied for optimization. In our implementation, we resort to
alternating direction method (ADM) for solving this subproblem because of
its reported efficiency~\citep{Boyd-ADM-2010,Yuan-ADM-2012}. The
implementation details of ADM for solving~\eqref{subprob:omega_F}
are deferred to Appendix~\ref{append:adm}. The modified GraHTP for
the precision matrix estimation problem is formally described in
Algorithm~\ref{alg:GraHTP_Modified}.

Generally speaking, the guarantees in Theorem~\ref{thrm:convergence}
and Theorem~\ref{thrm:recovery_GraHTP} are not valid for the
modified GraHTP. However, if $\bar\Omega - \alpha I$ and $\beta I
-\bar\Omega$ are diagonally dominant, then with a slight
modification of proof, we can prove that the Part (a) of
Theorem~\ref{thrm:recovery_GraHTP} is still valid for the modified
GraHTP. We sketchily describe the proof idea as follows: Let
$Z:=[\bar\Omega]_{F^{(t)}}$. Since $\bar\Omega - \alpha I$ and
$\beta I -\bar\Omega$ are diagonally dominant, we have $Z - \alpha
I$ and $\beta I - Z$ are also diagonally dominant and thus $\alpha I
\preceq Z \preceq \beta I$. Since $\Omega^{(t)}$ is the minimum of
$L(\Omega)$ restricted over the union of the cone $\alpha I \preceq
\Omega \preceq \beta I$ and the supporting set $F^{(t)}$, we have
$\langle\nabla L(\Omega^{(t)}), Z - \Omega^{(t)}\rangle \ge 0$. The
remaining of the arguments follows that of the Part(a) of
Theorem~\ref{thrm:recovery_GraHTP}.

\begin{algorithm} \label{alg:GraHTP_Modified}
\SetKwInOut{Input}{Input}\SetKwInOut{Output}{Output}\SetKw{Initialization}{Initialization:}

\Initialization{$\Omega^{(0)}$ with $\Omega^{(0)} \succ 0$ and
$\|(\Omega^{(0)})^-\|_0\le 2k$ and $\alpha I \preceq \Omega^{(0)}
\preceq \beta I$ (typically $\Omega^{(0)}=\alpha I$), $t=1$.}

\Output{$\Omega^{(t)}$.}

\Repeat{halting condition holds} {

(S1) Compute $\tilde \Omega^{(t)} = \Omega^{(t-1)} - \eta \nabla
L(\Omega^{(t-1)})$;

(S2) Let $\tilde F^{(t)} = \supp((\tilde \Omega^{(t)})^-,2k)$ be the
indices of $(\tilde \Omega^{(t)})^-$ with the largest $2k$ absolute
values and $F^{(t)}= \tilde F^{(t)}\cup \{(1,1),...,(p,p)\}$;

(S3) Compute $\Omega^{(t)} = \argmin\left\{ L(\Omega), \alpha I
\preceq \Omega \preceq \beta I, \supp(\Omega) \subseteq
F^{(t)}\right\}$;

$t=t+1$;

}

\caption{A Modified GraHTP for Sparse Precision Matrix Estimation.}
\end{algorithm}

\section{Experimental Results}
\label{sect:experiments}

This section is devoted to show the empirical performances of GraHTP
and FGraHTP when applied to sparse logistic regression and sparse
precision matrix estimation problems. Here we do not report the
results of our algorithms in compressive sensing tasks because in
these tasks GraHTP and FGraHTP reduce to the well studied
HTP~\citep{Foucart-2011} and IHT~\citep{Blumensath-2009},
respectively. Our algorithms are implemented in Matlab 7.12 running
on a desktop with Intel Core i7 3.2G CPU and 16G RAM.

\subsection{Sparsity-Constrained $\ell_2$-Regularized Logistic Regression.}

\subsubsection{Monte-Carlo Simulation}

We consider a synthetic data model identical to the one used
in~\citep{Bahmani-2013}. The sparse parameter $\bar w$ is a $p =
1000$ dimensional vector that has $\bar k=100$ nonzero entries drawn
independently from the standard Gaussian distribution. Each data
sample is an independent instance of the random vector $u$ generated
by an autoregressive process $ [u]_{i+1} = \rho [u]_i +
\sqrt{1-\rho^2} [a]_i$ with $[u]_1 \sim \mathcal {N}(0,1)$, $[a]_i
\sim \mathcal {N}(0,1)$, and $\rho=0.5$ being the correlation. The
data labels, $v \in \{-1,1\}$, are then generated randomly according to
the Bernoulli distribution
\[
\mathbb{P}(v = 1|u; \bar w) = \frac{\exp (2 \bar w^\top u)}{1+\exp
(2 \bar w^\top u)}.
\]
We fix the regularization parameter $\lambda=10^{-4}$ in the
objective of~\eqref{prob:logisic}. We are interested in the
following {two cases}:
\begin{enumerate}
  \item \textbf{Case 1}: Cardinality $k$ is fixed and sample size $n$ is varying: we test with $k=100$ and $n \in\{100, 200,..., 2000\}$.
  \item \textbf{Case 2}: Sample size $n$ is fixed and cardinality $k$ is varying: we test with $n=500$ and $k\in\{100,150,...,500\}$.
\end{enumerate}
For each case, we compare GraHTP and FGraHTP with two
state-of-the-art greedy selection methods:
GraSP~\citep{Bahmani-2013} and FBS (Forward Basis
Selection)~\citep{Yuan-2013}. As aforementioned, GraSP is also a hard-thresholding-type
method. This method simultaneously selects at each iteration $k$
nonzero entries and update their values via exploring the top $k$ entries in the previous iterate as well as the top $2k$
entries in the previous gradient. FBS is a forward-selection-type
method. This method iteratively selects an atom from the dictionary
and minimizes the objective function over the linear combinations of
all the selected atoms. Note that all the considered algorithms have
geometric rate of convergence. We will compare the computational
efficiency of these methods in our empirical study. We initialize
$w^{(0)} = 0$. Throughout our experiment, we set the stopping
criterion as $\|w^{(t)} - w^{(t-1)}\|/\|w^{(t-1)}\|\le 10^{-4}$.\\

\noindent\textbf{Results.} Figure~\ref{fig:solution_lr}(a)
presents the estimation errors of the considered algorithms. From
the left panel of Figure~\ref{fig:solution_lr}(a) (i.e., Case
1) we observe that: (i) when cardinality $k$ is fixed, the
estimation errors of all the considered algorithms tend to decrease
as sample size $n$ increases; and (ii) in this case GraHTP and
FGraHTP are comparable and both are superior to GraSP and FBS. From
the right panel of Figure~\ref{fig:solution_lr}(a) (i.e., Case
2) we  observe that: (i) when $n$ is fixed, the estimation errors
of all the considered algorithms but FBS tend to increase as $k$
increases (FBS is relatively insensitive to $k$ because it is a
forward selection method); and (ii) in this case GraHTP and FGraHTP
are comparable and both are superior to GraSP and FBS at relatively
small $k<200$. Figure~\ref{fig:solution_lr}(b) shows the CPU
 times of the considered algorithms. From this group of
results we  observe that in most cases, FBS is the fastest one
while GraHTP and FGraHTP are superior or comparable to GraSP in
computational time. We also observe that when $k$ is relatively
small, GraHTP is even faster than FGraHTP although FGraHTP is
cheaper in per-iteration overhead. This is partially because when
$k$ is small, GraHTP tends to need  fewer iterations
than FGraHTP to converge.

\begin{figure}[h!]
\begin{center} \subfigure[Estimation
Error]{
\includegraphics[width=3in]{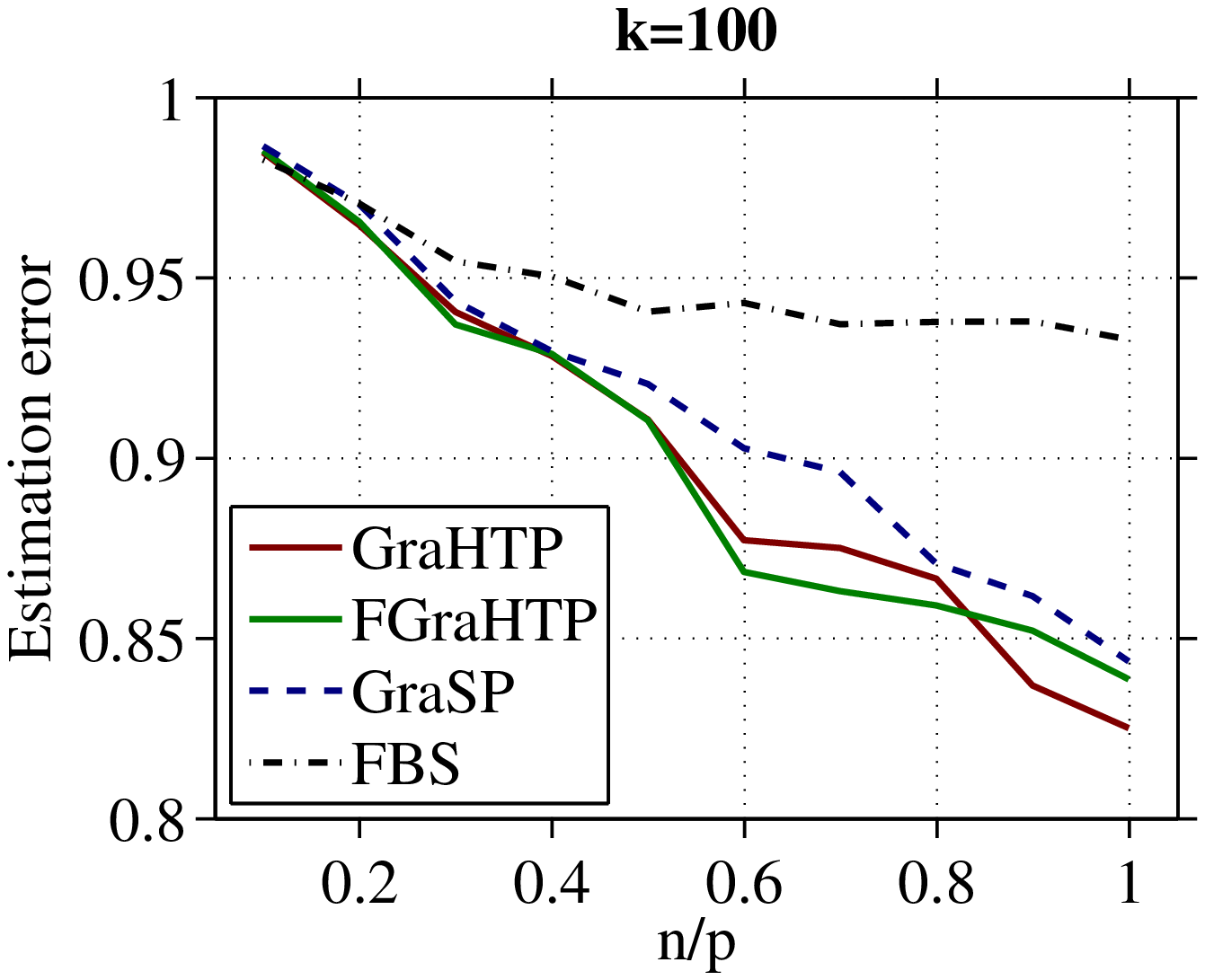}\hspace{0.2in}
\includegraphics[width=3in]{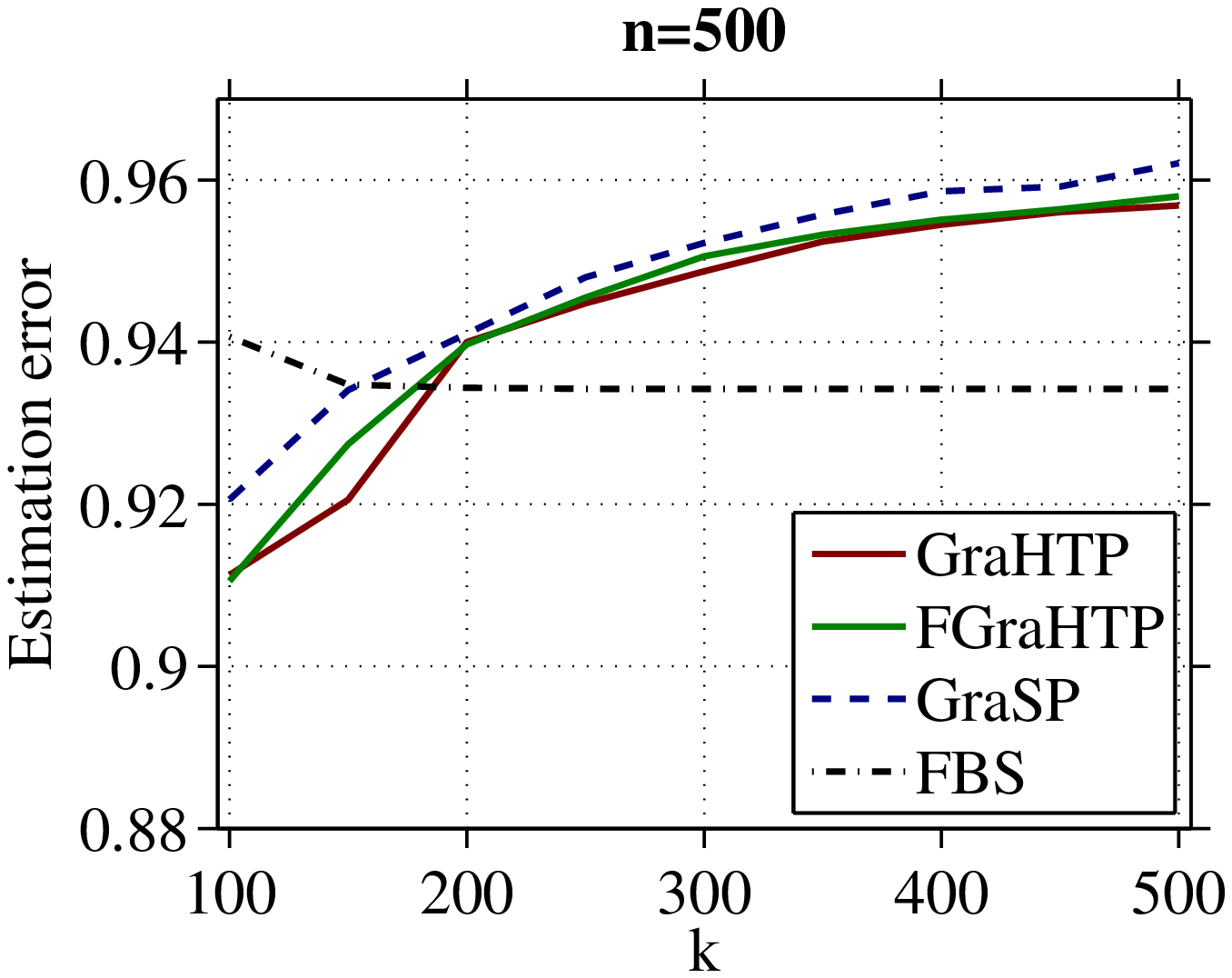}}\vspace{-0.1in}
\subfigure[CPU Running Time ]
{
\includegraphics[width=3in]{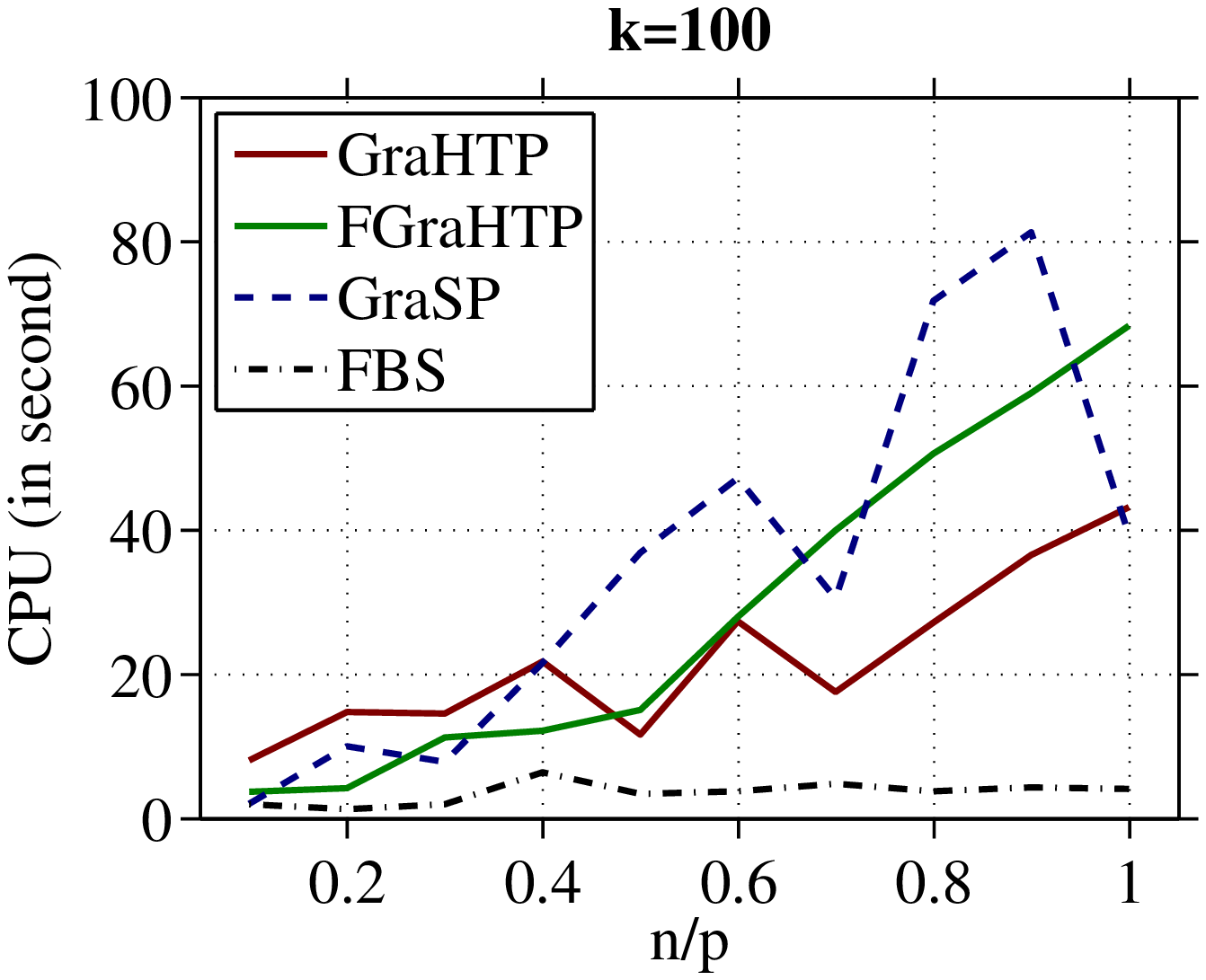}\hspace{0.2in}
\includegraphics[width=3in]{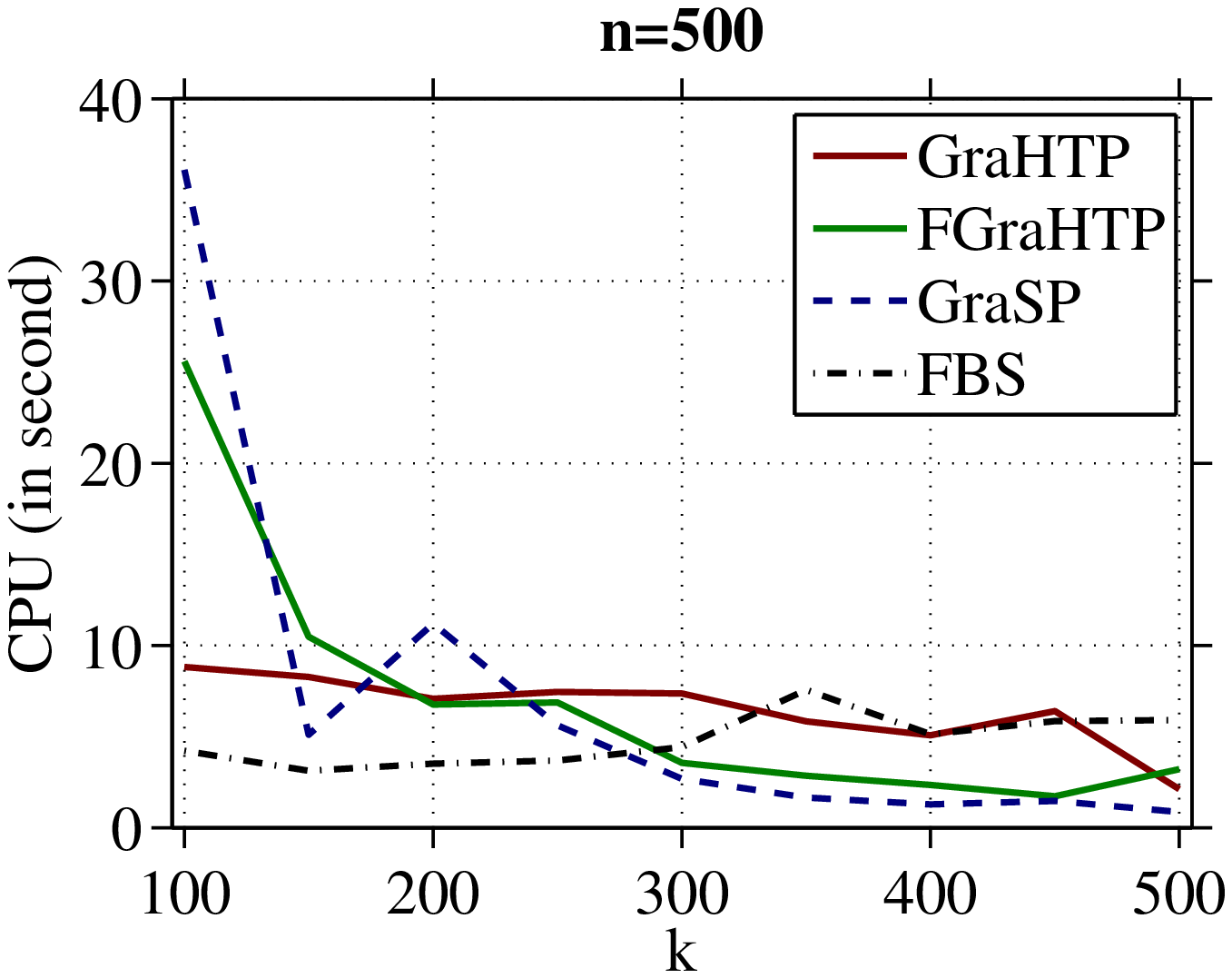}}
\end{center}
\vspace{-0.3in}
\caption{Simulation data: estimation error and CPU time of the considered
algorithms. }\label{fig:solution_lr}\vspace{-0.1in}
\end{figure}

\newpage

\subsubsection{Real Data}

The algorithms are also compared on the \textsf{rcv1.binary} dataset
($p$ = 47,236) which is a popular dataset for binary
classification on sparse data. A training subset of size $n$ =
20,242 and a testing subset of size 20,000 are used. We test with
sparsity parameters $k\in\{100, 200, ..., 1000\}$ and fix the
regularization parameter $\lambda = 10^{-5}$. The initial vector is
$w^{(0)} = 0$ for all the considered algorithms. We set the stopping
criterion as $\|w^{(t)} - w^{(t-1)}\|/\|w^{(t-1)}\|\le 10^{-4}$ or
the iteration stamp $t>50$.

Figure~\ref{fig:rcv1} shows the evolving curves of empirical
logistic loss for $k= 200, 400, 800, 1000$. It can be observed from
this figure that GraHTP and GraSP are comparable in terms of
convergence rate and they are superior to FGraHTP and FBS. The
testing classification errors and CPU running time of the considered
algorithms are provided in Figure~\ref{fig:rcv1_cpu_error}: (i) in
terms of accuracy, all the considered methods are comparable; and
(ii) in terms of running time, FGraHTP is the most efficient one; GraHTP is significantly faster than GraSP and FBS. The reason that FGraHTP runs fastest is because it has very low per-iteration cost although its convergence
curve is slightly less sharper than GraHTP and GraSP (see
Figure~\ref{fig:rcv1}). To summarize, GraHTP and FGraHTP achieve
better trade-offs between accuracy and efficiency than GraHTP and
FBS on \textsf{rcv1.binary} dataset.

\begin{figure}[h!]
\centering
\includegraphics[width=3in]{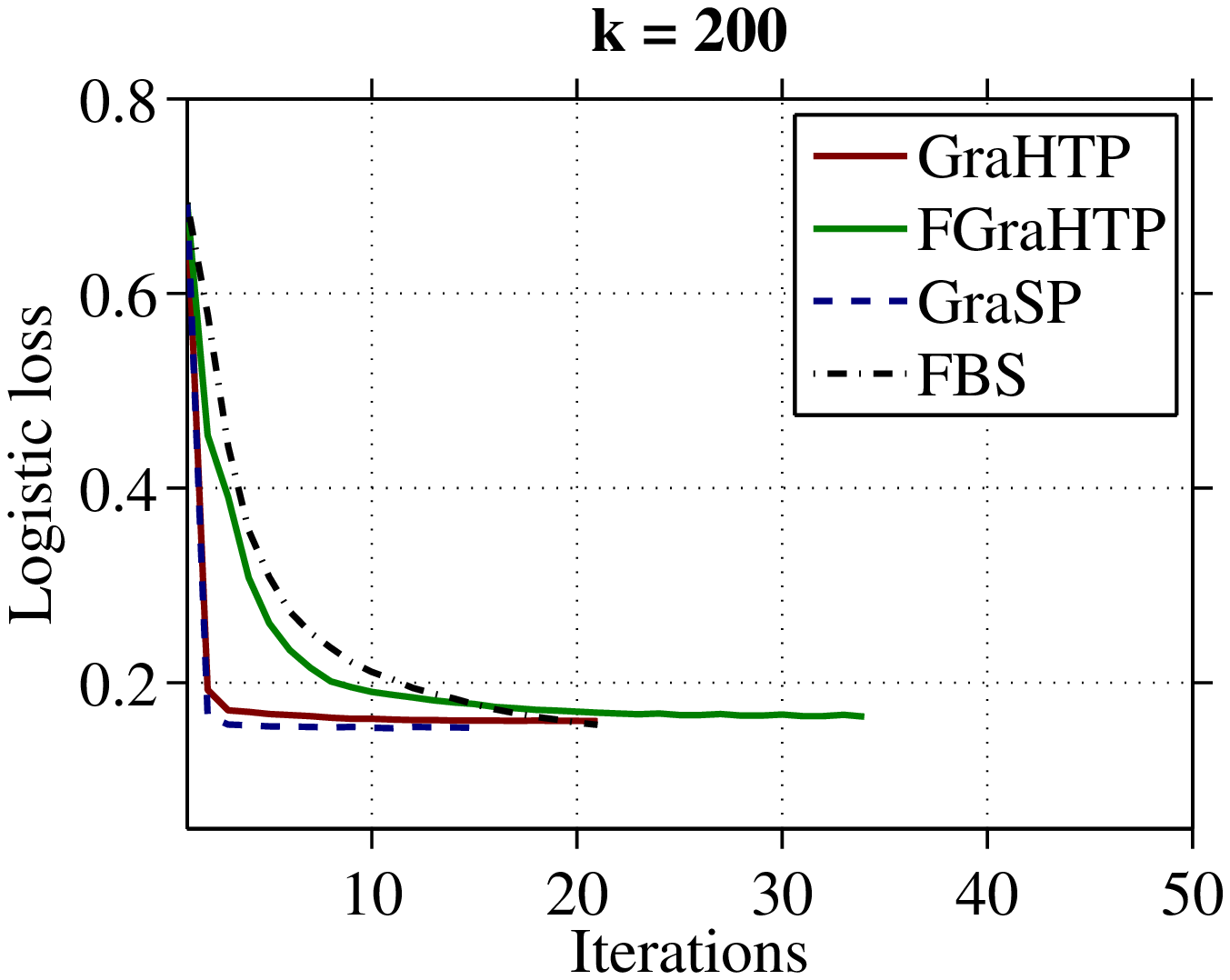}
\includegraphics[width=3in]{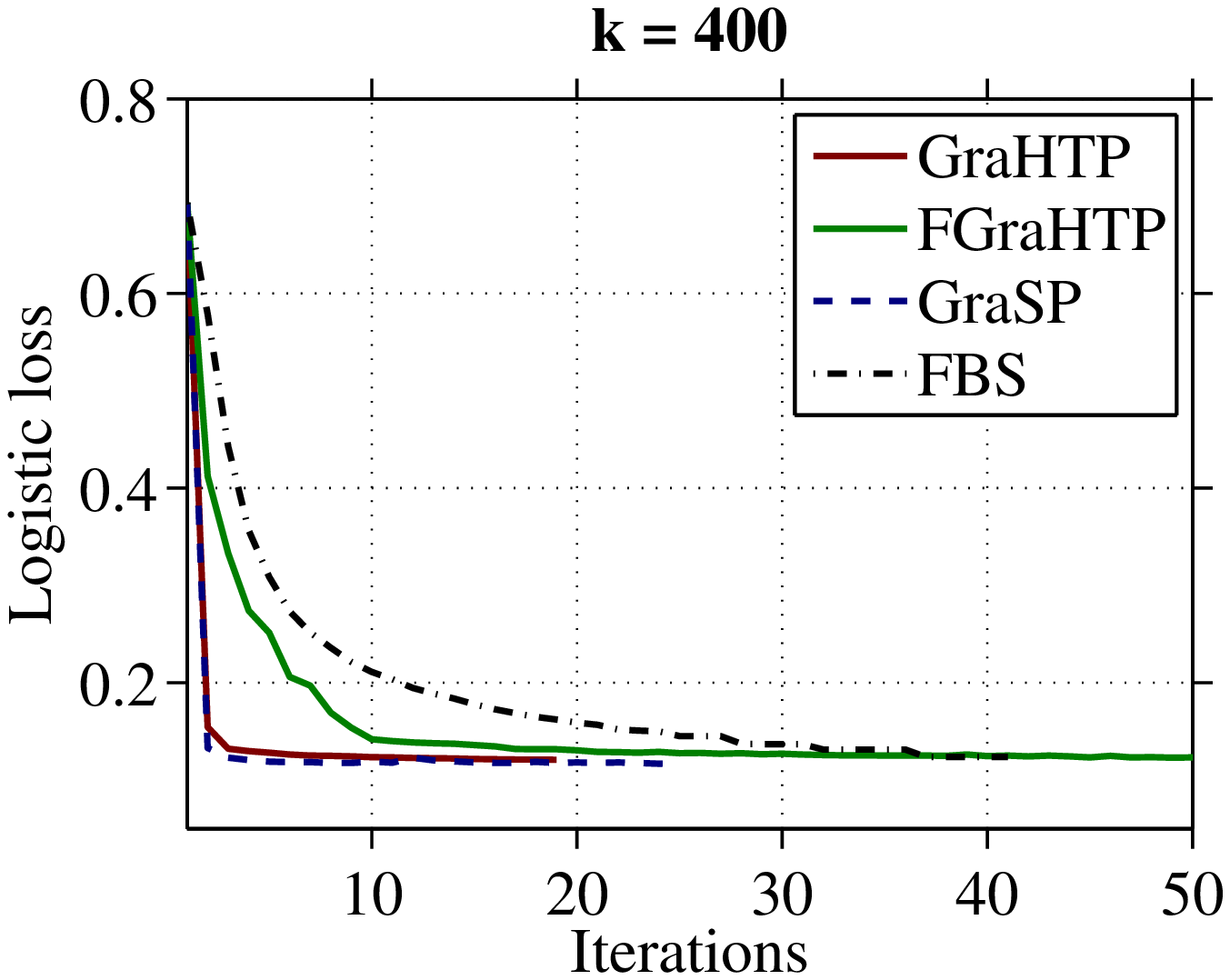}
\includegraphics[width=3in]{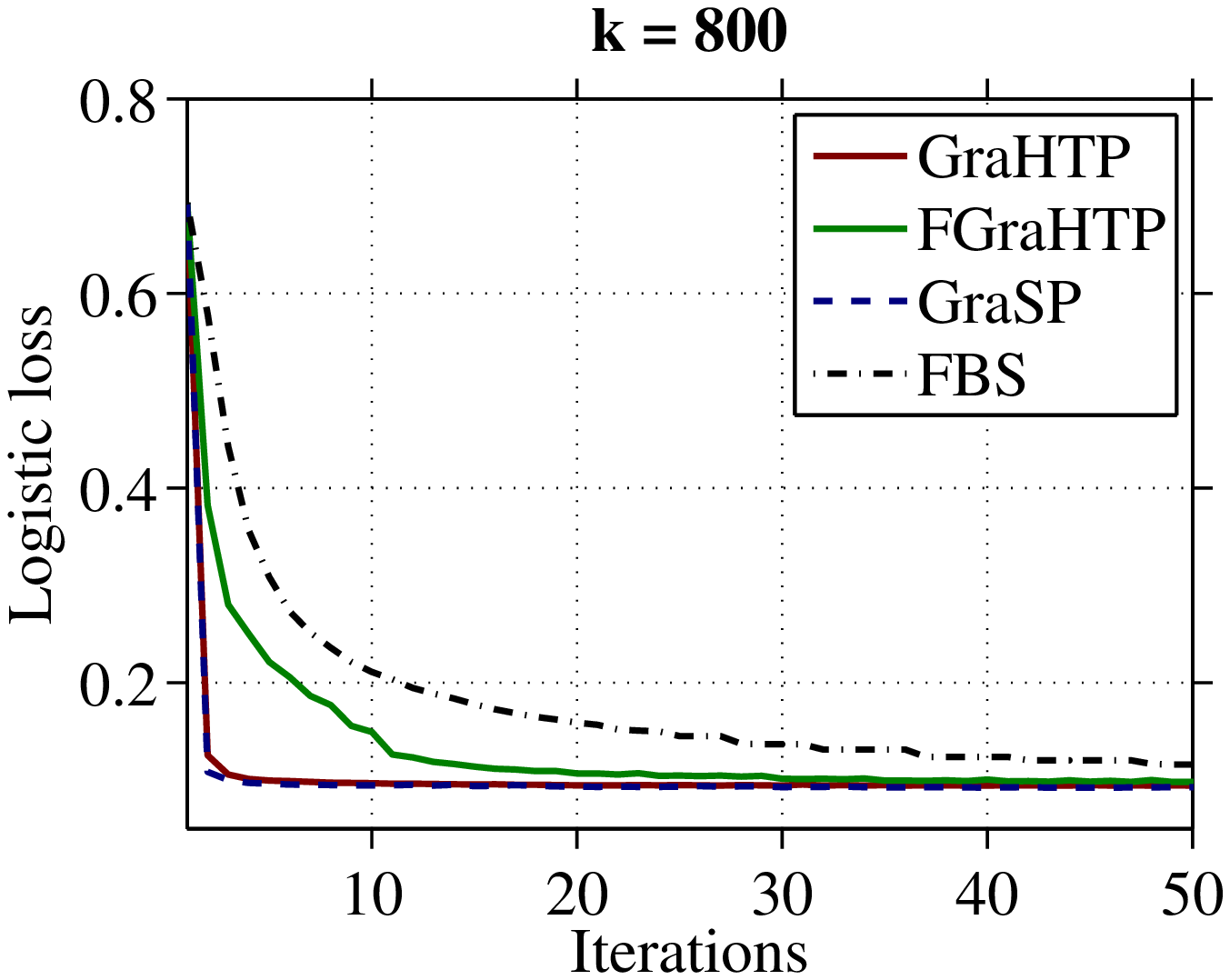}
\includegraphics[width=3in]{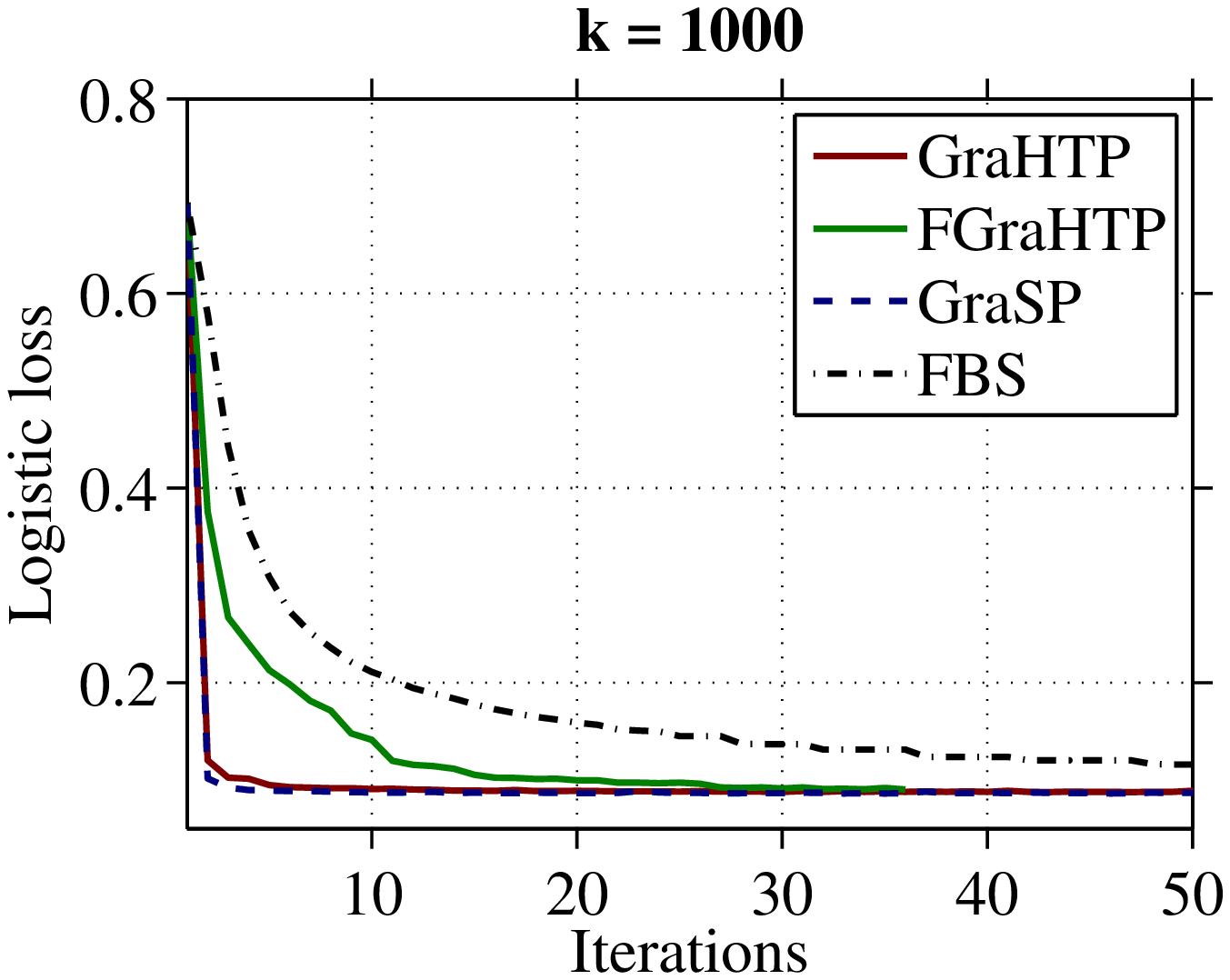}
\caption{\textsf{rcv1.binary} data: $\ell_2$-regularized logistic loss \emph{vs.}
number of iterations.}\label{fig:rcv1}
\end{figure}

\begin{figure}[h!]
\centering
\includegraphics[width=3in]{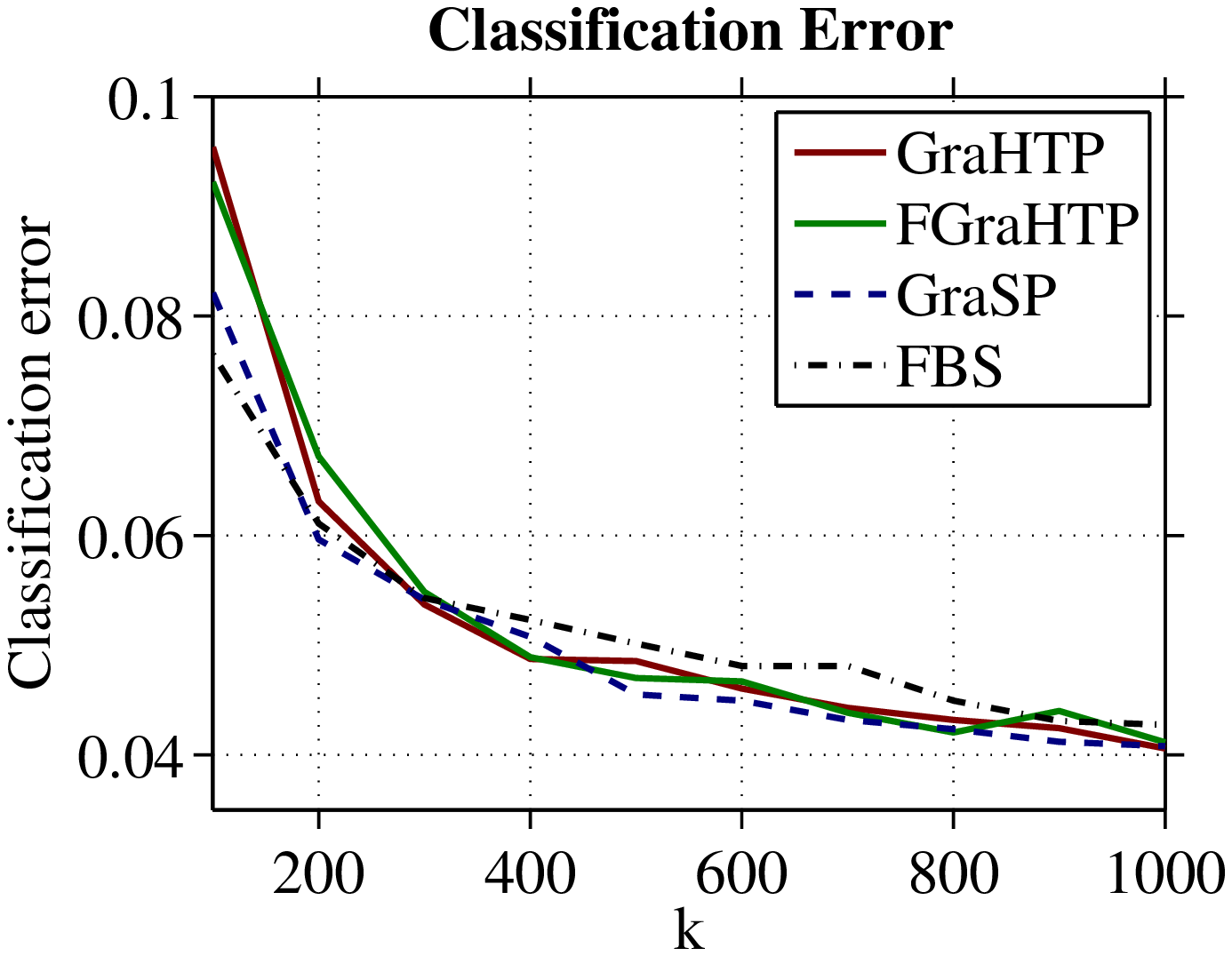}
\includegraphics[width=3in]{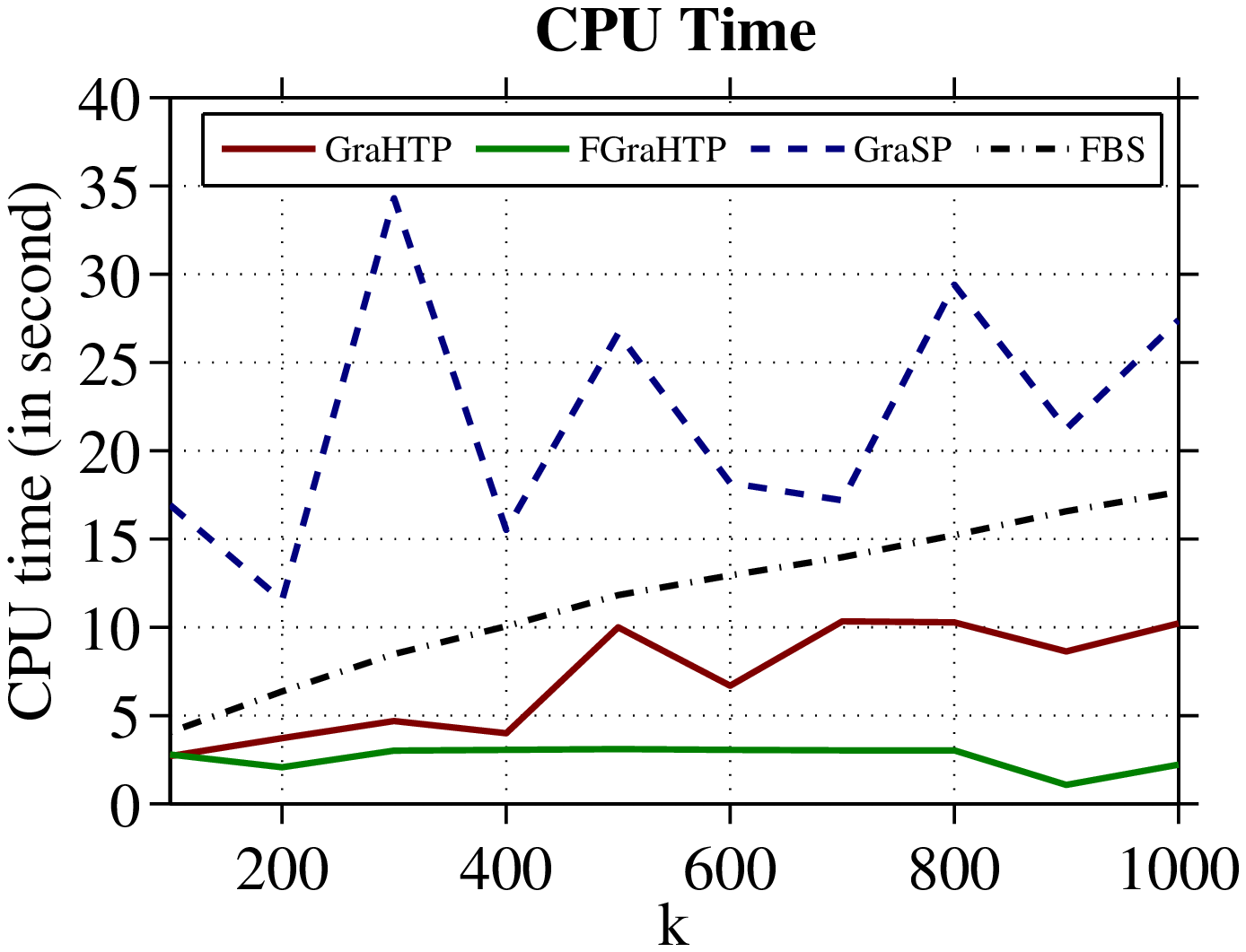}
\caption{\textsf{rcv1.binary} data: Classification error and CPU running time
curves of the considered methods.}\label{fig:rcv1_cpu_error}
\end{figure}

%\begin{table}[htb]
%\begin{center}
%\caption{rcv-1: Comparison of classification error and CPU time.
%\label{tab:rcv1_results}}
%\begin{tabular}{c c c c c}
%\hline
%Methods  & GraHTP  & FGraHTP & GraSP  & FBS \\
%\hline
%Classification Error   & 4.79 &  4.70 & \textbf{4.46} & 4.83 \\
%CPU Time (in Sec.)   & 10.22 & \textbf{3.07} & 28.68 & 12.76 \\
%
%\hline
%\end{tabular}
%\end{center}
%\end{table}

\newpage
%\clearpage

\subsection{Sparsity-Constrained Precision Matrix Estimation}

\subsubsection{Monte-Carlo Simulation}

Our simulation study employs the sparse precision matrix model
$\bar\Omega = B + \sigma I$ where each off-diagonal entry in $B$ is
generated independently and equals 1 with probability $P = 0.1$ or 0
with probability $1-P = 0.9$. $B$ has zeros on the diagonal, and
$\sigma$ is chosen so that the condition number of $\bar \Omega$ is
$p$. We generate a training sample of size $n=100$ from $\mathcal
{N}(0, \bar\Sigma)$, and an independent sample of size 100 from the
same distribution for tuning the parameter $k$. We compare
performance for different values of $p \in \{30, 60, 120, 200\}$,
replicated 100 times each.

We compare the modified GraHTP (see
Algorithm~\ref{alg:GraHTP_Modified}) with GraSP and FBS. To adopt
GraSP to sparse precision matrix estimation, we modify the algorithm with a
similar two-stage strategy as used in the modified GraHTP such that
it can handle the eigenvalue bounding constraint in addition to the
sparsity constraint. In the work of~\citet{Yuan-2013}, FBS has
already been applied to sparse precision matrix estimation. Also, we
compare GraHTP with GLasso (Graphical Lasso) which is a
representative convex method for $\ell_1$-penalized log-determinant
program~\citep{Friedman-Glasso-2008}. The quality of precision
matrix estimation is measured by its distance to the truth in
Frobenius norm and the support recovery F-score. The larger the
F-score, the better the support recovery performance. The numerical
values over $10^{-3}$ in magnitude are considered to be nonzero.

\begin{figure}[h!]
\begin{center}
\includegraphics[width=3in]{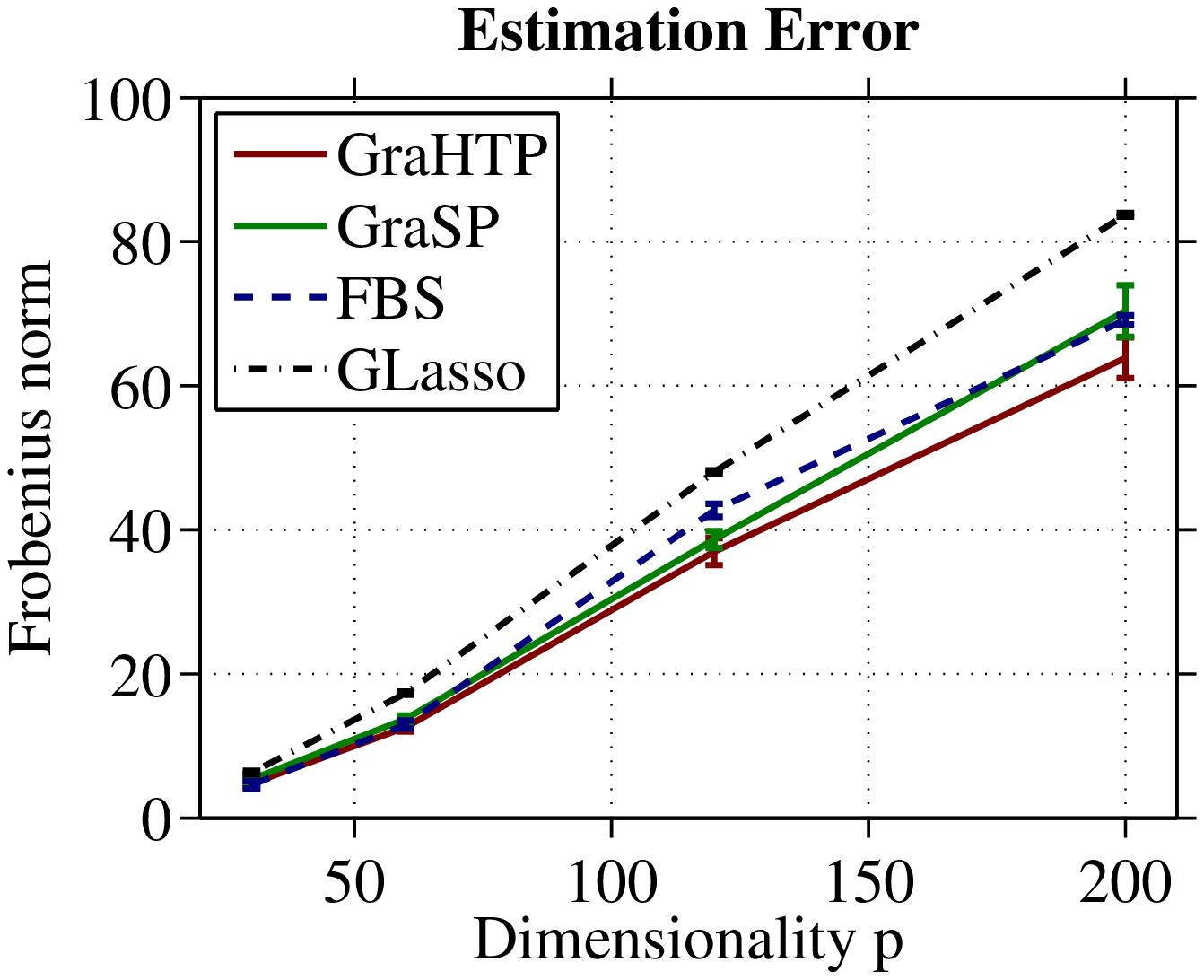}\label{fig:model3_frob}
\includegraphics[width=3in]{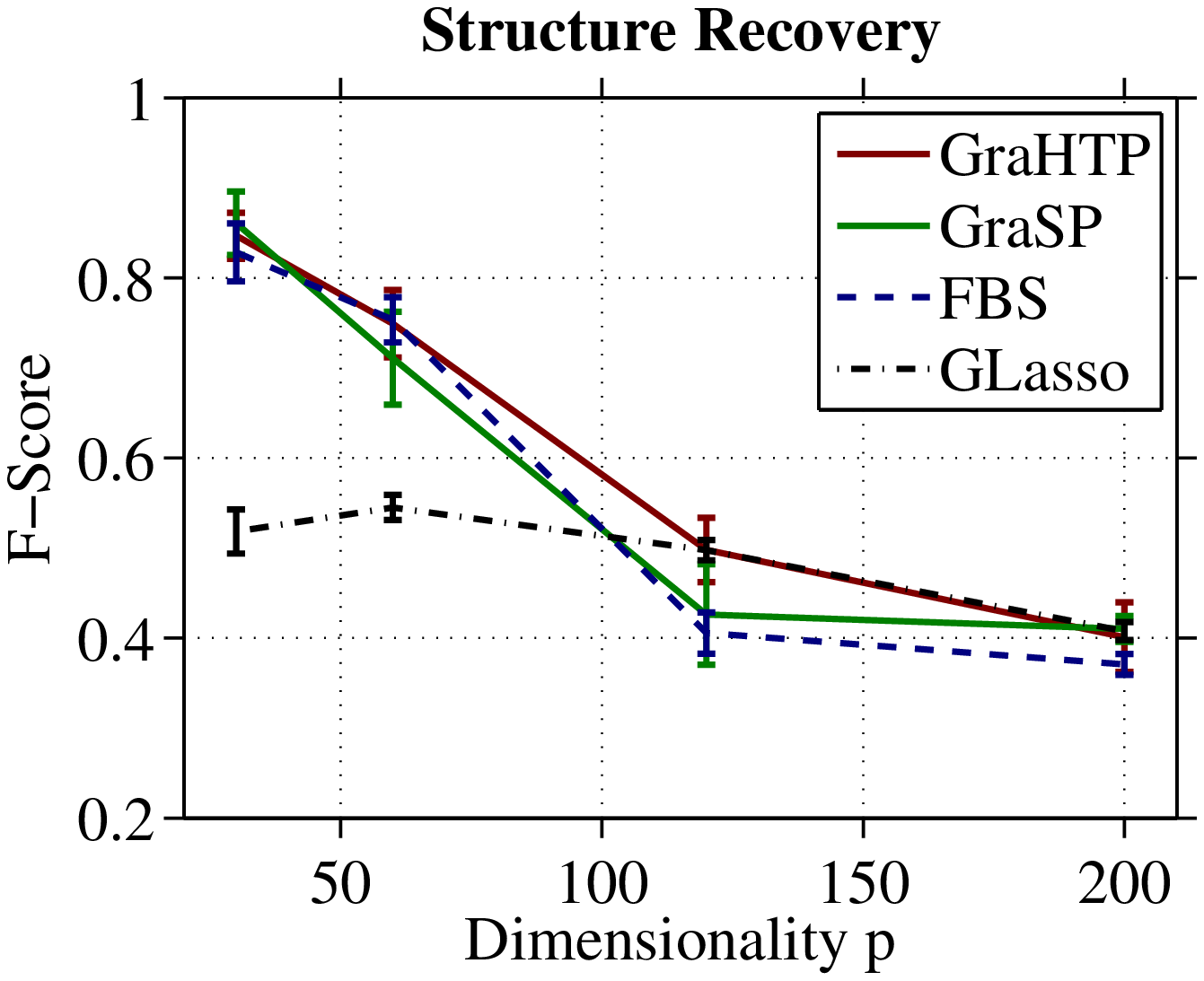}\label{fig:model3_fscore}
\includegraphics[width=3in]{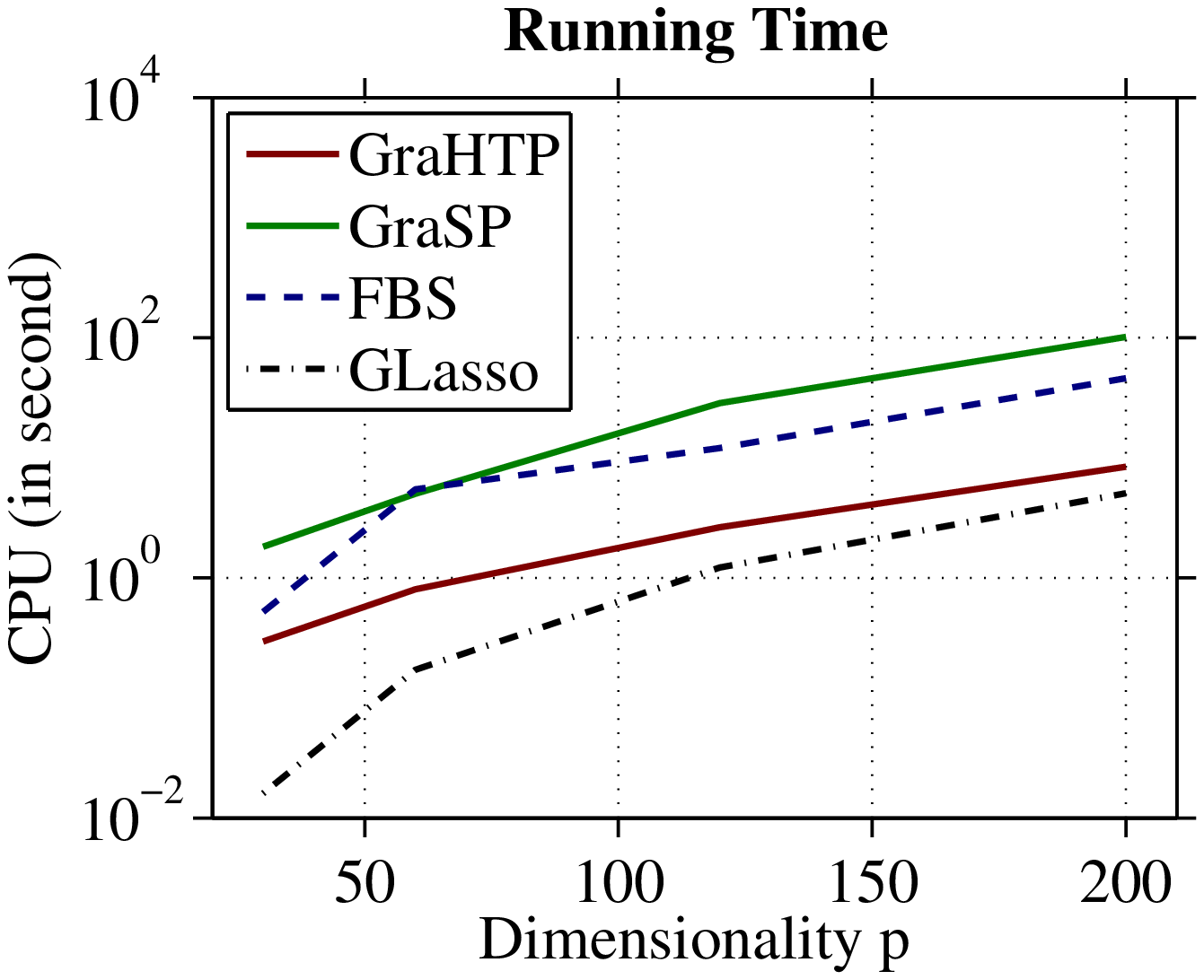}\label{fig:model3_cpu}
\end{center}
\vspace{-0.3in}
\caption{Simulated data: comparison of Matrix Frobenius norm loss,
support recovery F-score and CPU running time of the considered
methods.}\label{fig:model3}
\end{figure}

Figure~\ref{fig:model3} compares the matrix error in Frobenius norm,
sparse recovery F-score and CPU running time achieved by each of the
considered algorithms for different $p$. The results show that
GraHTP performs favorably in terms of estimation error and support
recovery accuracy. We note that the standard errors of GraHTP is
relatively larger than Glasso. This is as expected since GraHTP
approximately solves a non-convex problem via greedy selection at
each iteration; the procedure is less stable than convex methods
such as GLasso. Similar phenomenon of instability is observed for
GraSP and FBS. Figure~\ref{fig:model3_cpu} shows the computational
time of the considered algorithms. We observed that GLasso, as a
convex solver, is computationally superior to the three considered
greedy selection methods. Although inferior to GLasso, GraHTP is computationally
much more efficient than GraSP and FBS.

To visually inspect the support recovery performance of the
considered algorithms, we show in Figure~\ref{fig:synthetic_heatmap}
the heatmaps corresponding to the percentage of each matrix entry
being identified as a nonzero element with $p=60$. Visual inspection
on these heatmaps shows that the three greedy selection methods,
GraHTP, GraSP, FBS, tend to be sparser than GLasso. Similar phenomenon is
observed in our experiments with other values of $p$.

\begin{figure}[h!]
\begin{center} \subfigure[Ground
truth]{\includegraphics[width=50mm]{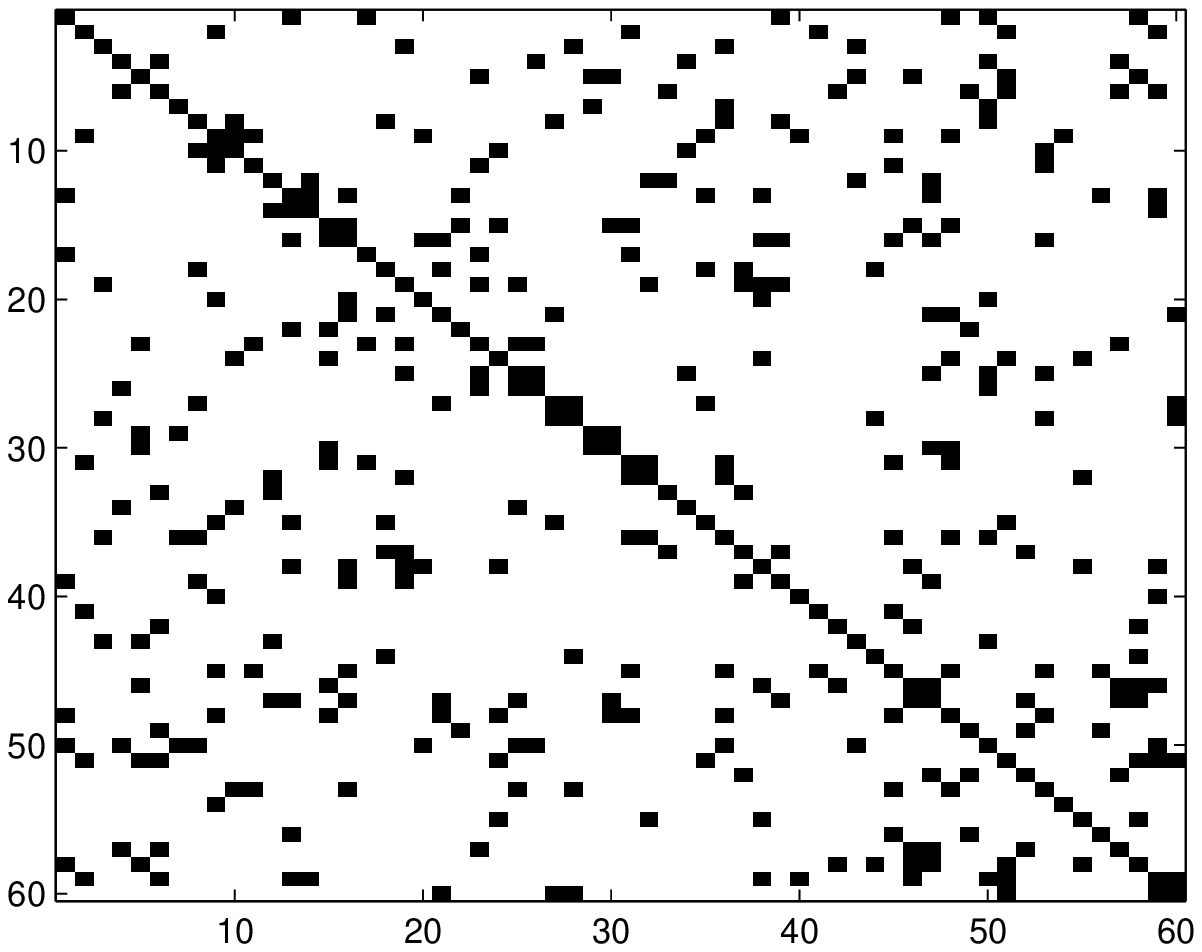}\label{fig:model_3_p_60_gnd}}
\subfigure[GraHTP]{\includegraphics[width=50mm]{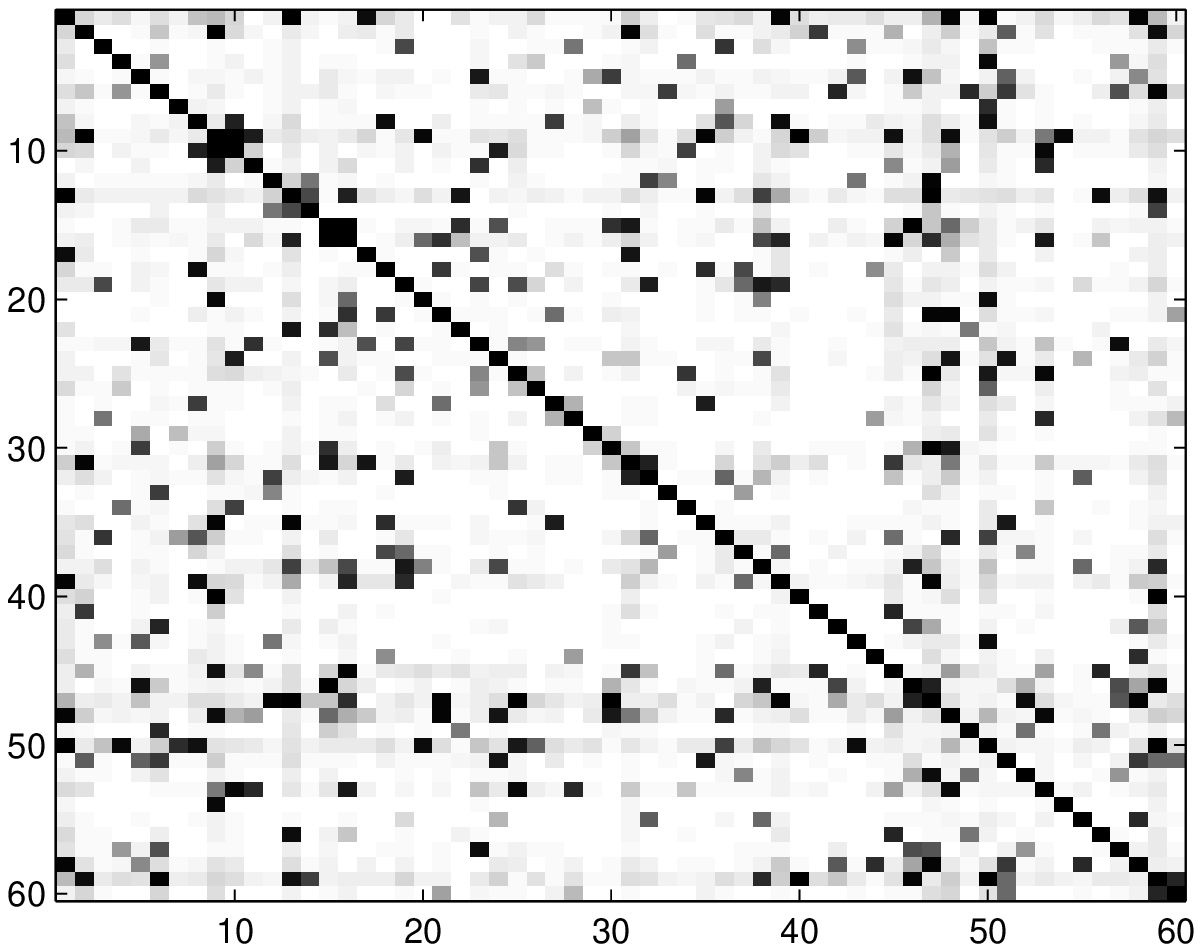}\label{fig:model_3_heatmap_img_p_60_NTGP}}
\subfigure[GraSP]{\includegraphics[width=50mm]{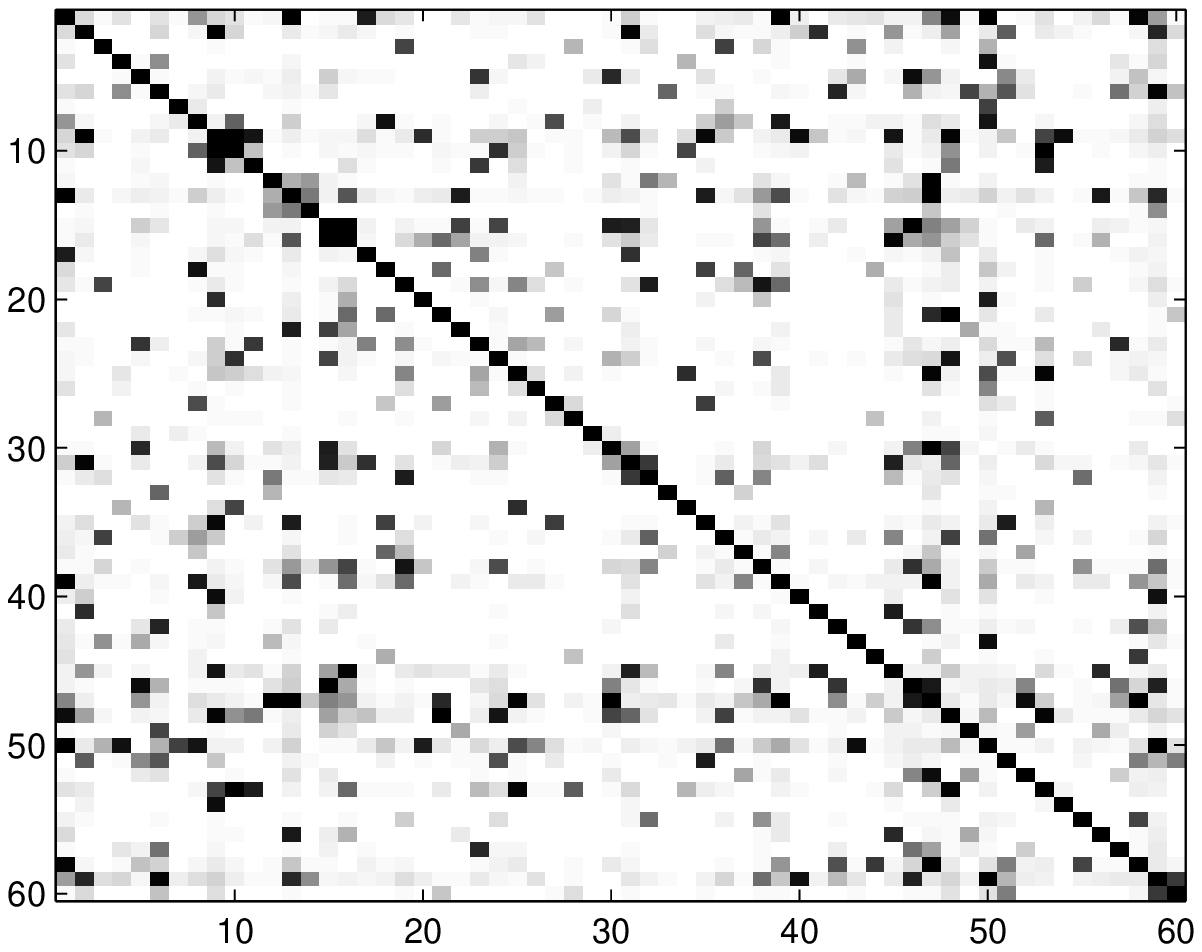}\label{fig:model_3_heatmap_img_p_60_GraSP}}
\subfigure[FBS]{\includegraphics[width=50mm]{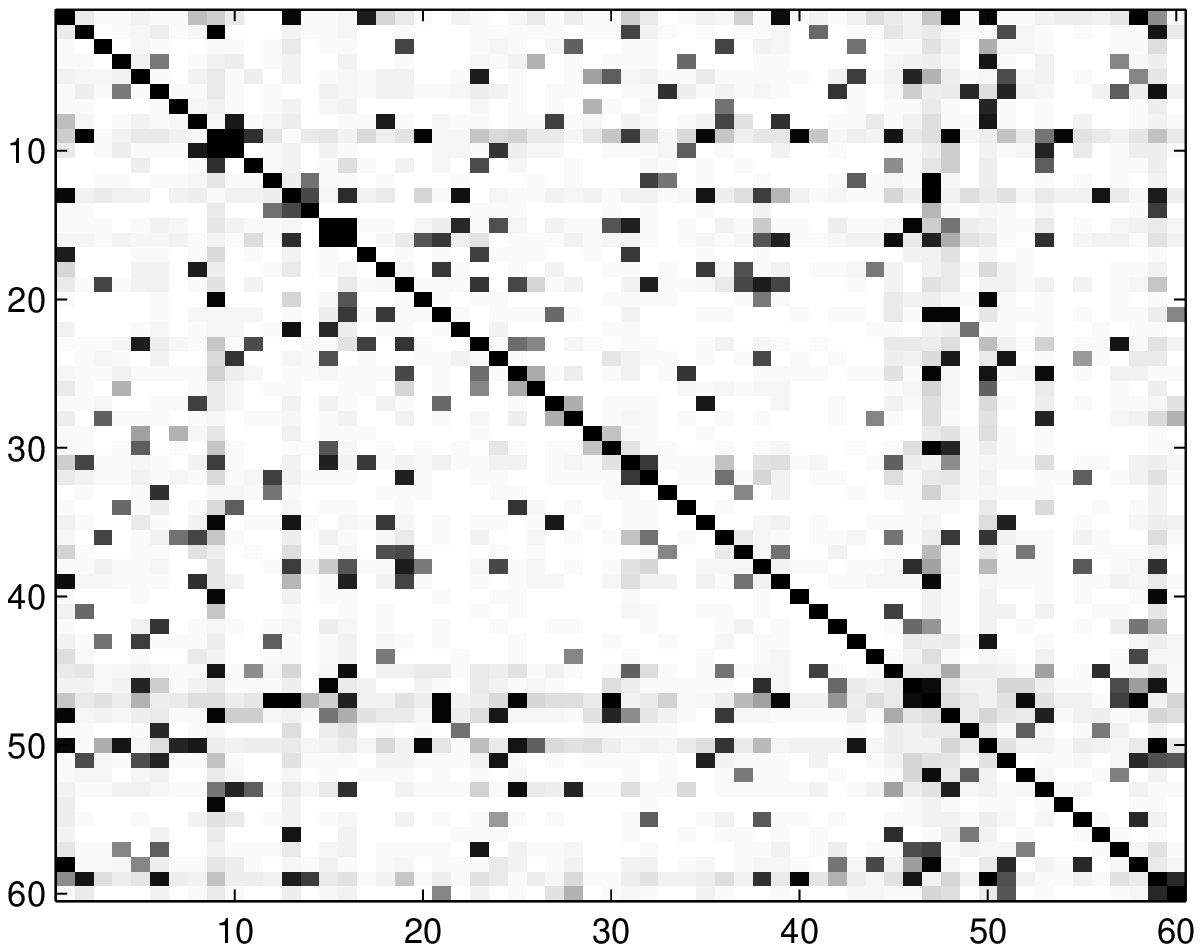}\label{fig:model_3_heatmap_img_p_60_FBS}}
\subfigure[GLasso]{\includegraphics[width=50mm]{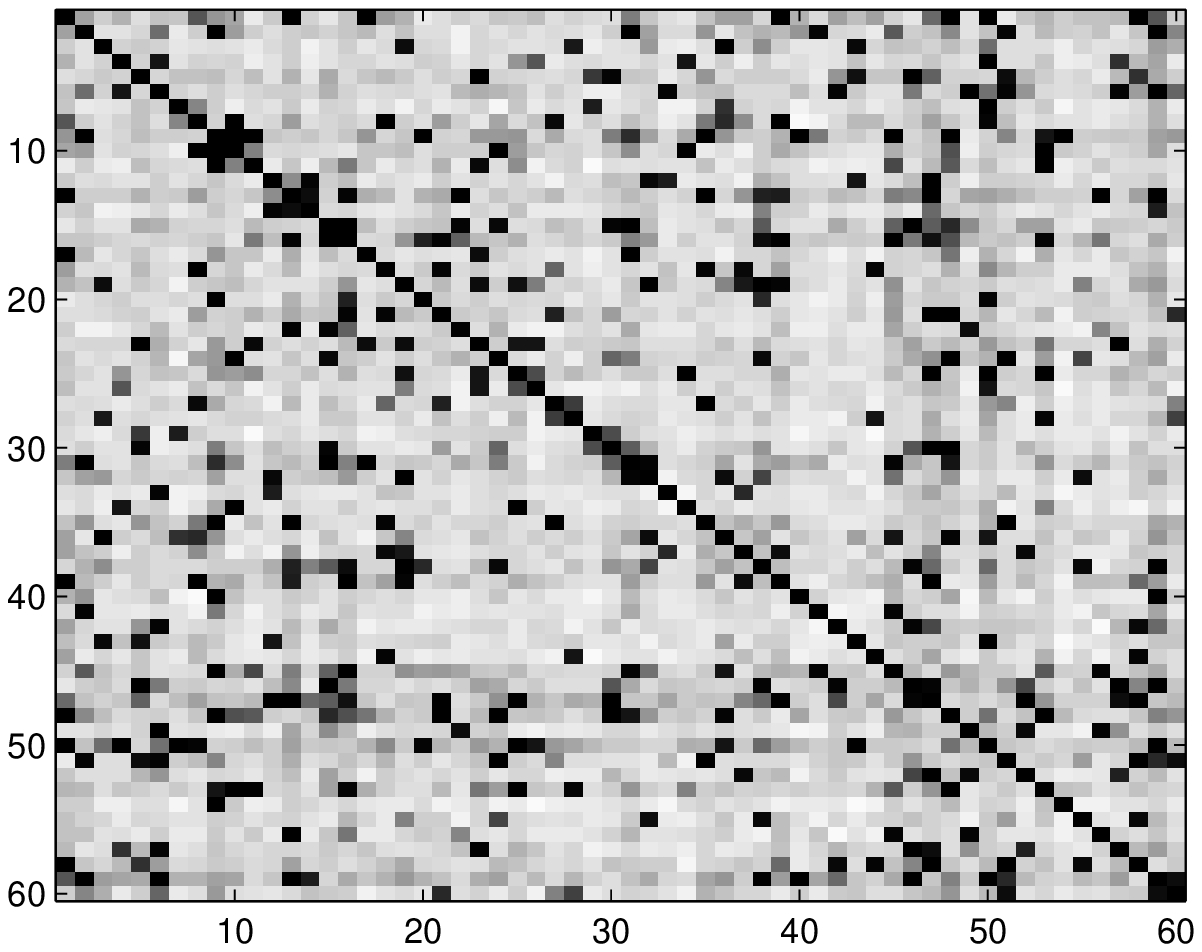}\label{fig:model_3_heatmap_img_p_60_GLasso}}
\end{center}
\vspace{-0.3in}
\caption{Simulation data with $p=60$: heatmaps of the frequency of each precision matrix entry being identified as
nonzeros out of 100 replications. \label{fig:synthetic_heatmap}}\hspace{-0.1in}
\end{figure}

%\newpage

\subsubsection{Real Data}

We consider the task of LDA (linear discriminant analysis) classification of tumors using the
\textsf{breast cancer} dataset~\citep{Hess-breastcancer-2006}. This dataset consists
of 133 subjects, each of which is associated with 22,283 gene expression levels. Among these subjects, 34
are with pathological complete response (pCR) and 99 are
with residual disease (RD). The pCR subjects are considered to have
a high chance of cancer free survival in the long term. Based on the
estimated precision matrix of the gene expression levels,
we apply LDA to predict whether a subject can achieve the pCR
state or the RD state.

In our experiment, we follow the same  protocol used
by~\citep{Cai-CLIME-2011} as well as references therein. For the sake of readers, we briefly review this experimental setup. The data are
randomly divided into the training and testing sets. In each random division, 5 pCR subjects
and 16 RD subjects are randomly selected to constitute the testing data, and the remaining subjects
form the training set with size $n=112$. By using two-sample $t$ test, $p=113$ most
significant genes are selected as covariates. Following the
LDA framework, we assume that the normalized gene expression data are
normally distributed as $\mathcal {N}(\mu_l, \bar\Sigma)$, where the
two classes are assumed to have the same covariance matrix,
$\bar\Sigma$, but different means, $\mu_l$, $l = 1$ for pCR state and $l =
2$ for RD state. Given a testing data sample $x$, we calculate its LDA scores, $\delta_l(x) = x^\top\hat\Omega {\hat\mu}_l - \frac{1}{2}
{\hat\mu}_l^\top\hat\Omega {\hat\mu}_l + \log {\hat\pi}_l$, $l=1,2$, using the precision matrix $\hat\Omega$ estimated by
the considered methods. Here $\hat\mu_l = (1/n_l)\sum_{i \in \text{class}_l} x_i$ is the
within-class mean in the training set and ${\hat\pi}_l = n_l /n$
is the proportion of class $l$ subjects in the training set. The
classification rule is set as $\hat{l}(x) = \argmax_{l=1,2}
\delta_l (x)$. Clearly, the classification performance is directly affected by the
estimation quality of $\hat{\Omega}$. Hence, we assess the precision matrix estimation performance on the testing data and compare (modified) GraHTP with GraSP and FBS. We also compare
GraHTP with GLasso (Graphical Lasso)~\citep{Friedman-Glasso-2008}. We use a 6-fold
cross-validation on the training data for tuning $k$. We repeat the
experiment 100 times.\\

\noindent\textbf{Results.} To compare the classification
performance, we use specificity, sensitivity (or recall), and
Mathews correlation coefficient (MCC) criteria as
in~\citep{Cai-CLIME-2011}:
\begin{align}\notag
&\text{Specificity} = \frac{\text{TN}}{\text{TN} + \text{FP}},\hspace{0.3in}
\text{Sensitivity} = \frac{\text{TP}}{\text{TP} +
\text{FN}},\\\notag
&\text{MCC} = \frac{\text{TP} \times \text{TN} -
\text{FP}\times\text{FN}}{\sqrt{(\text{TP} + \text{FP})(\text{TP} +
\text{FN})(\text{TN} + \text{FP})(\text{TN} + \text{FN})}},
\end{align}
where TP and TN stand for true positives (pCR) and true negatives
(RD), respectively, and FP and FN stand for false positives/
negatives, respectively. The larger the criterion value, the better the
classification performance. Since one can adjust decision threshold
in any specific algorithm to trade-off specificity and sensitivity
(increase one while reduce the other), the MCC is more meaningful as
a single performance metric.

\begin{table}[htb]
\begin{center}
\caption{\textsf{Brease cancer} data: comparison of average (std) classification accuracy and average CPU running time over 100 replications.
\label{tab:breast_cancer_results}}
\begin{tabular}{c c c c c}
\hline
Methods & Specificity & Sensitivity & MCC & CPU Time (sec.)\\
\hline
GraHTP     & 0.77 (0.11) & 0.77 (0.19) & \textbf{0.49} (0.19) & 1.92\\
GraSP    & 0.73 (0.10) & \textbf{0.78} (0.18) & 0.45 (0.17) & 4.06 \\
FBS     &  0.78 (0.11) &  0.74 (0.18)  &  0.48 (0.19) & 8.73 \\
GLasso   & \textbf{0.81} (0.11) & 0.64 (0.21) & 0.45 (0.19) & \textbf{1.19} \\
\hline
\end{tabular}
\end{center}\vspace{-0.1in}
\end{table}

Table~\ref{tab:breast_cancer_results}
reports the averages and standard deviations, in the parentheses, of the three classification criteria over 100 replications. It can be observed that GraHTP
is quite competitive to the leading methods in terms of the three
metrics. The averages of CPU running time (in seconds) of the considered
methods are listed in the last column of
Table~\ref{tab:breast_cancer_results}.
Figure~\ref{fig:breast_cancer} shows the evolving curves of
log-determinant loss verses number of iterations.
We observed that on this data, GraHTP converges much faster than GraSP and FBS.
Note that we did not draw the curve of GLasso in
Figure~\ref{fig:breast_cancer} because its objective function is
different from that of the
problem~\eqref{prob:SISC_card_constraint_bounds}.

\begin{figure}[h!]
\begin{center}
\includegraphics[width=3in]{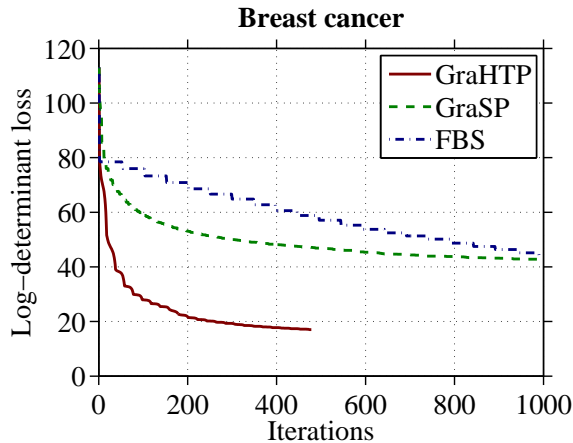}
\end{center}
\vspace{-0.3in} \caption{\textsf{Brease cancer} data: log-determinant loss
convergence curves of the considered methods.
\label{fig:breast_cancer}}\hspace{-0.1in}
\end{figure}

\section{Conclusion}
\label{sect:conclusion}

In this paper, we propose GraHTP as a generalization of HTP from compressive sensing to the generic problem of sparsity-constrained optimization. The main idea is to force the gradient descent iteration to be sparse via hard threshloding. Theoretically, we prove that under mild conditions, GraHTP converges geometrically in finite steps of iteration and its estimation error is controlled by the restricted norm of gradient at the target sparse solution. We also propose and analyze the FGraHTP algorithm as a fast variant of GraHTP without the debiasing step. Empirically, we compare GraHTP and FGraHTP with several representative greedy selection methods when applied to sparse logistic regression and sparse precision matrix estimation tasks. Our theoretical results and empirical evidences show that simply combing gradient descent with hard threshloding, with or without debiasing, leads to  efficient and accurate computational procedures for sparsity-constrained optimization problems.

\section*{Acknowledgment}

Xiao-Tong Yuan was a postdoctoral research associate supported by
NSF-DMS 0808864 and NSF-EAGER 1249316. Ping Li is supported by
ONR-N00014-13-1-0764, AFOSR-FA9550-13-1-0137, and NSF-BigData
1250914. Tong Zhang is supported by NSF-IIS 1016061, NSF-DMS 1007527, and NSF-IIS 1250985.

\appendix

\section{Technical Proofs}

\subsection{Proof of Lemma~\ref{lemma:strong_smooth}}\label{append:proof_lemma_strong_smooth}
\begin{proof}
\textbf{Part (a)}: The first inequality follows from the triangle
inequality. The second inequality can be derived by combining the
first one and the integration $f(x) - f(y) - \langle \nabla
f(y),x-y\rangle = \int_{0}^1 \langle \nabla f(y + \tau (x-y)) -
\nabla f(y), x-y \rangle d\tau$.

\textbf{Part (b)}: By adding two copies of the
inequality~\eqref{inequat:strong_smooth_convex} with $x$ and $y$
interchanged and using the Theorem 2.1.5 in~\citep{Nesterov-2004},
we know that
\[
(x-y)^\top (\nabla f(x) - \nabla f(y)) \ge m_{s}\|x-y\|^2, \quad
\|\nabla_F f(x) - \nabla_F f(y)\| \le M_{s}\|x-y\|.
\]
For any $\zeta >0$ we have
\[
\|x - y - \zeta \nabla_F f(x) + \zeta \nabla_F f(y)\|^2 \le (1- 2
\zeta m_s + \zeta^2 M_s^2)\|x-y\|^2.
\]
If $\zeta < 2m_s/M^2_s$, then $\rho_s = \sqrt{1-2\zeta m_s + \zeta
^2 M_s^2}<1$. This proves the desired result.
\end{proof}

\subsection{Proof of Theorem~\ref{thrm:convergence}}\label{append:proof_thrm_convergence}
\begin{proof}
We first prove the finite iteration guarantee of GraHTP. Let us
consider the vector $\tilde x^{(t)}_k$ which is the restriction of
$\tilde x^{(t)}$ on $F^{(t)}$. According to the definition of
$x^{(t)}$ we have $f(x^{(t)}) \le f(\tilde x^{(t)}_k)$. It follows
that
\begin{eqnarray}\label{inequat:f_t-f_t-1}
f(x^{(t)}) - f(x^{(t-1)}) &\le& f(\tilde x^{(t)}_k) - f(x^{(t-1)})
\nonumber \\
&\le& \langle \nabla f(x^{(t-1)}) , \tilde x^{(t)}_k - x^{(t-1)}
\rangle +
\frac{1+\rho_{2k}}{2\zeta} \|\tilde x^{(t)}_k -x^{(t-1)}\|^2 \nonumber \\
&\le& -\frac{1}{2\eta} \|\tilde x^{(t)}_k -x^{(t-1)}\|^2  +
\frac{1+\rho_{2k}}{2\zeta} \|\tilde x^{(t)}_k -x^{(t-1)}\|^2 \nonumber \\
&=& -\frac{\zeta - \eta (1+\rho_{2k})}{2\zeta\eta}\|\tilde x^{(t)}_k
-x^{(t-1)}\|^2,
\end{eqnarray}
where the second inequality follows from
Lemma~\ref{lemma:strong_smooth} and the third inequality follows from the
fact that $\tilde x^{(t)}_k$ is a better $k$-sparse approximation to
$\tilde x^{(t)}$ than $x^{(t-1)}$ so that $\|\tilde x^{(t)}_k -
\tilde x^{(t)}\|=\|\tilde x^{(t)}_k - x^{(t-1)} + \eta \nabla
f(x^{(t-1)})\|^2 \le\|x^{(t-1)}- x^{(t-1)} + \eta \nabla
f(x^{(t-1)})\|^2 = \|\eta \nabla f(x^{(t-1)})\|^2$, which implies
$2\eta \langle \nabla f(x^{(t-1)}) , \tilde x^{(t)}_k -x^{(t-1)}
\rangle  \le - \|\tilde x^{(t)}_k - x^{(t-1)}\|^2$. Since $\eta
(1+\rho_{2k}) < \zeta$, it follows that the sequence
$\{f(x^{(t)})\}$ is nonincreasing, hence it is convergent. Since it
is also eventually periodic, it must be eventually constant. In view
of~\eqref{inequat:f_t-f_t-1}, we deduce that $\tilde
x^{(t)}_k=x^{(t-1)}$, and in particular that $F^{(t)}=F^{(t-1)}$,
for $t$ large enough. This implies that $x^{(t)} = x^{(t-1)}$ for
$t$ large enough, which implies the desired result.

By noting that $x^{(t)}=\tilde x^{(t)}_k$ and from the
inequality~\eqref{inequat:f_t-f_t-1}, we immediately establish the
convergence of the sequence $\{f(x^{(t)})\}$ defined by FGraHTP.
\end{proof}

\subsection{Proof of Theorem~\ref{thrm:recovery_GraHTP}}\label{append:proof_thrm_recovery_grahtp}

\begin{proof}
\textbf{Part (a)}: The first step of the proof is a consequence of
the debiasing step S3. Since $x^{(t)}$ is the minimum of $f(x)$
restricted over the supporting set $F^{(t)}$, we have $\langle\nabla
f(x^{(t)}), z\rangle=0$ whenever $\supp(z) \subseteq F^{(t)}$. Let
$\bar F = \supp (\bar x)$ and $F = \bar F \cup F^{(t)}\cup
F^{(t-1)}$. It follows that
\begin{eqnarray}
\|(x^{(t)} - \bar x)_{F^{(t)}}\|^2 &= & \langle x^{(t)} - \bar x ,
(x^{(t)} - \bar x)_{F^{(t)}}\rangle \nonumber \\
&=&  \langle x^{(t)} - \bar x -\zeta \nabla_{F^{(t)}} f(x^{(t)}) +
\zeta\nabla_{F^{(t)}} f(\bar x), (x^{(t)} - \bar x)_{F^{(t)}}\rangle
- \zeta \langle \nabla_{F^{(t)}} f(\bar x) ,
(x^{(t)} - \bar x)_{F^{(t)}}\rangle\nonumber \\
&\le& \rho_s \|x^{(t)} - \bar x\|\|(x^{(t)} - \bar x)_{F^{(t)}}\|+
\zeta \|\nabla_k f(\bar x)\|\|(x^{(t)} - \bar x)_{F^{(t)}}\|
\nonumber,
\end{eqnarray}
where the last inequality is from Condition $C(s,\zeta,\rho_s)$,
$\rho_{k+\bar k} \le \rho_s$ and $\|\nabla_{F^{(t)}} f(\bar x)\| \le
\|\nabla_k f(\bar x)\|$. After simplification, we have $ \|(x^{(t)}
- \bar x)_{F^{(t)}}\| \le \rho_s \|x^{(t)} - \bar x\| + \zeta
\|\nabla _k f(\bar x)\|$. It follows that
\begin{eqnarray}
\|x^{(t)} - \bar x\| \le \|(x^{(t)} - \bar x)_{F^{(t)}}\| +
\|(x^{(t)} - \bar x)_{\overline{F^{(t)}}}\| \le \rho_s \|x^{(t)} -
\bar x\| + \zeta \|\nabla _k f(\bar x)\| + \|(x^{(t)} - \bar
x)_{\overline{F^{(t)}}}\|. \nonumber
\end{eqnarray}
After rearrangement we obtain
\begin{equation}\label{inequat:recovery_grahtp_1}
\|x^{(t)} - \bar x\| \le \frac{\|(x^{(t)} - \bar
x)_{\overline{F^{(t)}}}\|}{1-\rho_s } + \frac{\zeta\|\nabla _k
f(\bar x)\|}{1-\rho_s}.
\end{equation}

The second step of the proof is a consequence of steps S1 and S2. We
notice that
\[
\|(x^{(t-1)} - \eta \nabla f(x^{(t-1)}))_{\bar F}\| \le \|(x^{(t-1)}
- \eta \nabla f(x^{(t-1)}))_{F^{(t)}}\|.
\]
By eliminating the contribution on $\bar F \cap F^{(t)}$, we derive
\[
\|(x^{(t-1)} - \eta \nabla f(x^{(t-1)}))_{\bar F\backslash F^{(t)}}
\| \le \|(x^{(t-1)} - \eta \nabla f(x^{(t-1)}))_{F^{(t)}\backslash
\bar F}\|.
\]
For the right-hand side, we have
\[
\|(x^{(t-1)} - \eta \nabla f(x^{(t-1)}))_{F^{(t)}\backslash \bar
F}\| \le \|(x^{(t-1)} - \bar x - \eta \nabla f(x^{(t-1)}) + \eta
\nabla f(\bar x) )_{F^{(t)}\backslash \bar F}\| + \eta \|\nabla_k
f(\bar x)\|.
\]
As for the left-hand side, we have
\begin{eqnarray}
&&\|(x^{(t-1)} - \eta \nabla f(x^{(t-1)}))_{\bar
F\backslash F^{(t)}}\| \nonumber \\
&\ge& \|(x^{(t-1)} - \bar x - \eta \nabla f(x^{(t-1)}) + \eta \nabla
f(\bar x) )_{\bar F \backslash F^{(t)}} + (x^{(t)} - \bar
x)_{\overline{F^{(t)}}}\| - \eta \|\nabla_k f(\bar x)\| \nonumber\\
&\ge& \|(x^{(t)} - \bar x)_{\overline{F^{(t)}}}\| - \|(x^{(t-1)} -
\bar x - \eta \nabla f(x^{(t-1)}) + \eta \nabla f(\bar x) )_{\bar F
\backslash F^{(t)}}\| - \eta \|\nabla_k f(\bar x)\|.\nonumber
\end{eqnarray}
With $\bar F \Delta F^{(t)}$ denoting the symmetric difference of
the set $\bar F$ and $F^{(t)}$, it follows that
\begin{eqnarray}\label{inequat:recovery_grahtp_2}
&&\|(x^{(t)} - \bar x)_{\overline{F^{(t)}}}\| \nonumber \\
&\le& \sqrt{2}\|(x^{(t-1)} - \bar x - \eta \nabla f(x^{(t-1)}) +
\eta \nabla f(\bar x))_{\bar F \Delta F^{(t)}}\| + 2\eta \|\nabla_k
f(\bar x)\|\nonumber \\
&\le& \sqrt{2} \|x^{(t-1)} - \bar x - \eta \nabla_F f(x^{(t-1)}) +
\eta \nabla_F f(\bar x)\| + 2\eta \|\nabla_k
f(\bar x)\|\nonumber \\
&\le& \sqrt{2} \|x^{(t-1)} - \bar x - \zeta \nabla_F f(x^{(t-1)}) + \zeta \nabla_F f(\bar x)\| + \sqrt{2} (\zeta -\eta)\|\nabla_F f(x^{(t-1)}) - \nabla_F f(\bar x)\|+ 2\eta \|\nabla_k f(\bar x)\|\nonumber \\
&\le& \sqrt{2}(1-\eta/\zeta + (2 - \eta/\zeta)\rho_s) \|x^{(t-1)} -
\bar x\| + 2\eta \|\nabla_k f(\bar x)\|,
\end{eqnarray}
where the last inequality follows from Condition $C(s,\zeta,\rho_s)$,
$\eta< \zeta$ and Lemma~\ref{lemma:strong_smooth}. As a final step,
we put~\eqref{inequat:recovery_grahtp_1}
and~\eqref{inequat:recovery_grahtp_2} together to obtain
\[
\|x^{(t)} - \bar x\| \le \frac{\sqrt{2}(1-\eta/\zeta +
(2-\eta/\zeta)\rho_s)}{1-\rho_s } \|x^{(t-1)} - \bar x\| +
\frac{(2\eta+\zeta)\|\nabla _k f(\bar x)\|}{1-\rho_s}
\]
Since $\mu_1 = \sqrt{2}(1-\eta/\zeta +
(2-\eta/\zeta)\rho_s)/(1-\rho_s)<1$, by recursively applying the
above inequality we obtain the desired inequality in part (a).\\

\textbf{Part (b):} Recall that $F^{(t)} = \supp(x^{(t)},k)$ and
$F=F^{(t-1)} \cup F^{(t)} \cup \supp(\bar{x})$. Consider the
following vector
\[ y = x^{(t-1)} - \eta \nabla_F f(x^{(t-1)}).
\]
By using triangular inequality we have
\begin{eqnarray*}
\|y -\bar x\| &=& \|x^{(t-1)} - \eta \nabla_F f(x^{(t-1)}) - \bar
x\| \nonumber \\
&\le& \|x^{(t-1)} -\bar x -
\eta \nabla_F f(x^{(t-1)}) + \eta \nabla_F f(\bar x)\|+ \eta\|\nabla_F f(\bar x)\| \\
&\le& (1-\eta/\zeta + (2-\eta/\zeta)\rho_s)\|x^{(t-1)}-\bar x\| +
\eta\|\nabla_s f(\bar x)\|,
\end{eqnarray*}
where the last inequality follows from Condition $C(s,\zeta,\rho_s)$, $\eta < \zeta$
and $\|\nabla_F f(\bar x)\| \le \|\nabla_s f(\bar x)\|$. For
FGraHTP, we note that $x^{(t)} = \tilde x^{(t)}_k = y_k$, and thus
$\|x^{(t)} - \bar x\| \le \|x^{(t)} - y\|+ \|y -\bar x\| \le 2\|y
-\bar x\| $. It follows that
\begin{eqnarray*}
\|x^{(t)} - \bar x\| \le  2(1-\eta/\zeta + (2-\eta/\zeta)\rho_s)
\|x^{(t-1)} - \bar x\| + 2 \eta\|\nabla_s f(\bar x)\|.
\end{eqnarray*}
Since $\mu_2 = 2(1-\eta/\zeta + (2-\eta/\zeta)\rho_s)<1$, by
recursively applying the above inequality we obtain the desired
inequality in part (b).
\end{proof}

\subsection{Proof of Proposition~\ref{prop:logistic_assump1}}
\label{append:proof_prop_logistic_assump1}

\begin{proof}
Obviously, $f(w)$ is $\lambda$-strongly convex. Consider an index
set $F$ with cardinality $|F|\le s$ and all $w,w'$ with
$\supp(w)\cup \supp(w')\subseteq F$. Since $\sigma(z)$ is Lipschitz
continuous with constant $1$, we have
\begin{eqnarray}
|[a(w)]_i - [a(w')]_i)| &=& 2|\sigma(2v^{(i)} w^\top u^{(i)}) -
\sigma(2v^{(i)} w'^\top u^{(i)})| \nonumber \\
&\le& 4|(w-w')^\top v^{(i)} u^{(i)}| \nonumber\\
&\le& 4\|(u^{(i)})_F\|\|w-w'\| \nonumber \\
&\le& 4R_s \|w-w'\|, \nonumber
\end{eqnarray}
which implies
\[
\|a(w) - a(w')\|_\infty \le 4R_s \|w-w'\|.
\]
Therefore we have
\begin{eqnarray}
\|\nabla_F f(w) - \nabla_F f(w')\| &\le& \frac{1}{n}
\|U_{F\bullet} (a(w) - a(w')) \| + \lambda\|w-w'\| \nonumber \\
&\le& \frac{1}{n}
\|U_{F\bullet} (a(w) - a(w')) \|_1 + \lambda\|w-w'\| \nonumber \\
&\le& \frac{1}{n} |U_{F\bullet}|_1\|a(w) - a(w')\|_\infty + \lambda\|w - w'\| \nonumber \\
&\le& (4 \sqrt{s} R_s^2 + \lambda) \|w-w'\|, \nonumber
\end{eqnarray}
where the second ``$\le$'' follows $\|x\|\le \|x\|_1$, the third
``$\le$'' follows from $\|Ax\|_1 \le |A|_1\|x\|_\infty$, and last
``$\le$'' follows from $|U_{F\bullet}|_1 \le n\sqrt{s}\max_i
\|[u^{(i)}]_F\| \le n\sqrt{s} R_s$. Therefore $f$ is $(4 \sqrt{s}
R_s^2 + \lambda)$-strongly smooth. The desired result follows
directly from  Part(b) of Lemma~\ref{lemma:strong_smooth}.
\end{proof}

\subsection{Proof of Proposition~\ref{prop_logsitic}}
\label{append:proof_prop_logsitic}

\begin{proof}
For any index set $F$ with $|F|\le s$, we can deduce
\begin{eqnarray}\label{proof:prop_logistic_1}
\|\nabla_F f(\bar w)\| &\le& \|[\nabla l(\bar w)]_F \| +
\lambda\|\bar w_F\| \le \sqrt{s}\|\nabla l(\bar w)\|_\infty +
\lambda \|\bar w_s\|.
\end{eqnarray}
We next bound the term $\|\nabla l(\bar w)\|_\infty$.
From~\eqref{equat:derivatives} we have
\begin{eqnarray}
\left|\frac{\partial l}{\partial [\bar w]_j}\right| &=&
\left|\frac{1}{n}\sum_{i=1}^n -v^{(i)} [u^{(i)}]_j +
\mathbb{E}_{v}[v[u^{(i)}]_j \mid
u^{(i)}]\right| \nonumber \\
&\le& \left|\frac{1}{n}\sum_{i=1}^n  v^{(i)} [u^{(i)}]_j -
\mathbb{E}[v[u]_j] \right| + \left| \frac{1}{n}\sum_{i=1}^n
\mathbb{E}_{v}[v[u^{(i)}]_j \mid u^{(i)}] - \mathbb{E}[v[u]_j]
\right| \nonumber,
\end{eqnarray}
where $\mathbb{E}[]$ is taken over the
distribution~\eqref{equat:joint_distr}. Therefore, for any
$\varepsilon >0 $,
\begin{eqnarray}
\mathbb{P}\left(\left|\frac{\partial l}{\partial [\bar w]_j}\right|
> \varepsilon \right) &\le&
\mathbb{P}\left(\left|\frac{1}{n}\sum_{i=1}^n  v^{(i)} [u^{(i)}]_j -
\mathbb{E}[v[u]_j] \right| > \frac{\varepsilon}{2}\right) \nonumber \\
&& + \mathbb{P}\left(\left| \frac{1}{n}\sum_{i=1}^n
\mathbb{E}_{v}[v[u^{(i)}]_j \mid u^{(i)}] - \mathbb{E}[v[u]_j]
\right|> \frac{\varepsilon}{2} \right) \nonumber \\
&\le&  4 \exp\left\{-\frac{n\varepsilon^2}{8\sigma^2}\right\}
\nonumber,
\end{eqnarray}
where the last ``$\le$'' follows from the large deviation inequality
of sub-Gaussian random variables which is standard~\citep[see,
e.g.,][]{Vershynin-SubExp-2011}. By the union  bound we have
\begin{eqnarray}
\mathbb{P}(\|\nabla l(\bar w)\|_\infty > \varepsilon) &\le&
4p\exp\left\{-\frac{n\varepsilon^2}{8\sigma^2}\right\} \nonumber.
\end{eqnarray}
By letting  $\varepsilon = 4\sigma\sqrt{\ln p/n}$, we know that
with probability at least $1-4p^{-1}$,
\[
\|\nabla l(\bar w)\|_\infty \le 4\sigma \sqrt{\ln p/n}.
\]
Combing the above inequality with~\eqref{proof:prop_logistic_1} yields
 the desired result.
\end{proof}

\section{Solving Subproblem~\eqref{subprob:omega_F} via ADM}
\label{append:adm}

In this appendix section, we provide our implementation details of
ADM for solving the subproblem~\eqref{subprob:omega_F}. By
introducing an auxiliary variable $\Theta \in \mathbb{R}^{p \times
p}$, the problem~\eqref{subprob:omega_F} is obviously equivalent to
the following problem:
\begin{equation}\label{prob:SISC_card_constraint_adm}
\min_{\alpha I \preceq \Omega \preceq \beta I} L(\Omega), \ \ \ \st
\Omega = \Theta, \supp(\Theta) \subseteq F.
\end{equation}
Then, the Augmented Lagrangian function
of~\eqref{prob:SISC_card_constraint_adm} is
\begin{equation}
J(\Omega,\Theta,\Gamma) := L(\Omega) - \langle \Gamma, \Omega -
\Theta \rangle + \frac{\rho}{2} \|\Omega - \Theta\|_{Frob}^2,
\nonumber
\end{equation}
where $\Gamma \in \mathbb{R}^{p \times p}$ is the multiplier of the
linear constraint $\Omega = \Theta$  and $\rho>0$ is the penalty
parameter for the violation of the linear constraint. The ADM solves
the following problems to generate the new iterate:
\begin{eqnarray}
\Omega^{(\tau+1)}&=& \argmin_{\alpha I \preceq \Omega \preceq \beta
I} J(\Omega,
\Theta^{(\tau)}, \Gamma^{(\tau)}), \label{subprob:x_k+1} \\
\Theta^{(\tau+1)} &=& \argmin_{\supp(\Theta) \subseteq F} J(\Omega^{(\tau+1)}, \Theta, \Gamma^{(\tau)}), \label{subprob:y_t+1_tgd}\\
\Gamma^{(\tau+1)} &=& \Gamma^{(\tau)} - \rho(\Omega^{(\tau+1)} -
\Theta^{(\tau+1)}). \nonumber
\end{eqnarray}
Let us first consider the minimization
problem~\eqref{subprob:x_k+1}. It is easy to verify that it is
equivalent to the following minimization problem:
\begin{equation}
\Omega^{(\tau+1)} = \argmin_{\alpha I \preceq \Omega \preceq \beta
I} \frac{1}{2}\left\|\Omega - M\right\|_{Frob}^2 - \frac{1}{\rho}
\log\det \Omega, \nonumber
\end{equation}
where
\[
M = \Theta^{(\tau)} - \frac{1}{\rho}(\Sigma_n - \Gamma^{(\tau)}).
\]
Let the singular value decomposition of $M$ be
\[
M = V \Lambda V^\top, \ \ \text{ with } \Lambda = \diag(\lambda_1,
..., \lambda_n).
\]
It is easy to verify that the solution of
problem~\eqref{subprob:x_k+1} is given by
\[
\Omega^{(\tau+1)} = V \tilde\Lambda V^\top, \ \ \text{ with }
\tilde\Lambda = \diag(\tilde\lambda_1,...,\tilde\lambda_n),
\]
where
\[
\tilde\lambda_j = \min\left\{\beta,\max\left\{\alpha,
\frac{\lambda_j + \sqrt{\lambda_j^2 + 4/\rho}}{2}\right\}\right\}.
\]
Next, we consider the minimization
problem~\eqref{subprob:y_t+1_tgd}. It is straightforward to see that
the solution of problem~\eqref{subprob:y_t+1_tgd} is given by
\[
\Theta^{(\tau+1)} = \left[\Omega^{(\tau+1)} -
1/\rho\Gamma^{(\tau)}\right]_F.
\]

%\newpage

\end{document}